\documentclass{article} 
\usepackage{iclr2022_conference,times}

\usepackage{microtype}
\usepackage{booktabs} 
\usepackage{amsmath,amsthm,verbatim,amssymb,amsfonts,amscd}
\usepackage{graphics}
\usepackage[colorlinks,linkcolor=blue,citecolor=blue]{hyperref}
\usepackage[bottom]{footmisc}
\usepackage{caption}
\usepackage{subcaption}
\usepackage{helvet}  
\usepackage{courier}  
\usepackage{url}  
\usepackage{graphicx}  
\usepackage{multirow}
\usepackage{amsthm}
\usepackage{color}
\usepackage{MnSymbol}
\usepackage{makecell}
\usepackage{arydshln}
\usepackage{amsmath}
\usepackage{xcolor}
\usepackage{natbib}
\usepackage{textcomp}
\usepackage{wrapfig}
\usepackage{algorithm}
\usepackage{algorithmic}
\usepackage{bm}
\usepackage{mathtools}
        \usepackage{setspace}

\def\multiset#1#2{\ensuremath{\left(\kern-.3em\left(\genfrac{}{}{0pt}{}{#1}{#2}\right)\kern-.3em\right)}}

\usepackage[utf8]{inputenc} 
\usepackage[T1]{fontenc}    
\usepackage{hyperref}       
\usepackage{url}            
\usepackage{booktabs}       
\usepackage{amsfonts}       
\usepackage{nicefrac}       
\usepackage{microtype}      
\usepackage{xcolor}         

\usepackage{hyperref}
\usepackage{url}

\theoremstyle{plain}
\newtheorem{theorem}{Theorem}
\newtheorem{corollary}{Corollary}
\newtheorem{lemma}{Lemma}
\newtheorem*{remark}{Remark}
\newtheorem{proposition}{Proposition}

\theoremstyle{definition}
\newtheorem{definition}{Definition}
\newtheorem*{theorem*}{Theorem}
\newtheorem*{proposition*}{Proposition}
\newtheorem*{corollary*}{Corollary}
\newtheorem{assumption}{Assumption}

\title{Strength of Minibatch Noise in SGD}

\author{Liu Ziyin$^*$, Kangqiao Liu$^*$, Takashi Mori, \& Masahito Ueda \\
The University of Tokyo
}

\iclrfinalcopy 
\begin{document}

\maketitle

\begin{abstract}
 The noise in stochastic gradient descent (SGD), caused by minibatch sampling, is poorly understood despite its practical importance in deep learning. This work presents the first systematic study of the SGD noise and fluctuations close to a local minimum. We first analyze the SGD noise in linear regression in detail and then derive a general formula for approximating SGD noise in different types of minima. For application, our results (1) provide insight into the stability of training a neural network, (2) suggest that a large learning rate can help generalization by introducing an implicit regularization, (3) explain why the linear learning rate-batchsize scaling law fails at a large learning rate or at a small batchsize and (4) can provide an understanding of how discrete-time nature of SGD affects the recently discovered power-law phenomenon of SGD.
\end{abstract}

\section{Introduction}
Stochastic gradient descent (SGD) is the simple and efficient optimization algorithm behind the success of deep learning \citep{allen2018convergence, xing2018_walk_sgd, Zhang_optimizationProperitesSGD,Wang2020,he2020recent, liu2021noise, Simsekli2019tailindex, wu2020noisy}. Minibatch noise, also known as the SGD noise, is the primary type of noise in the learning dynamics of neural networks. Practically, minibatch noise is unavoidable because a modern computer's memory is limited while the size of the datasets we use is large; this demands the dataset to be split into ``minibatches" for training. At the same time, using minibatch is also a recommended practice because using a smaller batch size often leads to better generalization performance \citep{Hoffer2017}. Therefore, understanding minibatch noise in SGD has been one of the primary topics in deep learning theory. Dominantly many theoretical studies take two approximations: (1) the continuous-time approximation, which takes the infinitesimal step-size limit; (2) the Hessian approximation, which assumes that the covariance matrix of the SGD noise is equal to the Hessian $H$. While these approximations have been shown to provide some qualitative understanding, the limitation of these approximations is not well understood. For example, it is still unsure when such approximations are valid, which hinders our capability to assess the correctness of the results obtained by approximations.

In this work, we fill this gap by deriving analytical formulae for discrete-time SGD with arbitrary learning rates and exact minibatch noise covariance. 
In summary, the main \textbf{contributions} are:
(1) we derive the strength and the shape of the minibatch SGD noise in the cases where the noise for discrete-time SGD is analytically solvable; (2) we show that the SGD noise takes a different form in different kinds of minima and propose general and more accurate approximations. This work is organized as follows: Sec.~\ref{sec: background} introduces the background. Sec.~\ref{sec: related works} discusses the related works. Sec.~\ref{sec:theory} outlines our theoretical results. Sec.~\ref{sec:general} derives new approximation formulae for SGD noises. In Sec.~\ref{sec: applications}, we show how our results can provide practical and theoretical insights to problems relevant to contemporary machine learning research. For reference, the relationship of this work to the previous works is shown in Table~\ref{tab: related works}.

\vspace{-3mm}
\section{Background}\label{sec: background}
\vspace{-2mm}
In this section, we introduce the minibatch SGD algorithm. Let $\{x_i, y_i\}_{i=1}^{N}$ be a training set. We can define the gradient descent (GD) algorithm for a differentiable loss function $L$ as  $\mathbf{w}_t = \mathbf{w}_{t-1} - \lambda \nabla_\mathbf{w}L(\mathbf{w},\{ \mathbf{x}, \mathbf{y}\})$, where $\lambda$ is the learning rate and $\mathbf{w}\in\mathbb{R}^{D}$ is the weights of the model. We consider an additive loss function for applying the minibatch SGD.
\begin{definition}
    A loss function $L(\{x_i, y_i\}_{i=1}^{N}, \mathbf{w})$ is additive if $ L(\{x_i, y_i\}_{i=1}^{N}, \mathbf{w}) =  \frac{1}{N}\sum_{i=1}^N \ell(x_i, y_i, \mathbf{w})$ for some differentiable, non-negative function $\ell(\cdot)$.
\end{definition}
This definition is quite general. Most commonly studied and used loss functions are additive, e.g., the mean-square error (MSE) and cross-entropy loss.
For an additive loss, the minibatch SGD with momentum algorithm can be defined.
\begin{definition}\label{def: with replacement}
    The \textit{minibatch SGD with momentum} algorithm by sampling {with replacement} computes the update to the parameter $\mathbf{w}$ with the following set of equations:\vspace{-1mm}
    \begin{equation}\label{eq: with replacement}
        \begin{cases}
            \hat{\mathbf{g}}_t= \frac{1}{S}\sum_{i\in B_t} \nabla  \ell(x_i, y_i, \mathbf{w}_{t-1});\\
            \mathbf{m}_t = \mu \mathbf{m}_{t-1} + \hat{\mathbf{g}}_t;\\
            \mathbf{w}_t = \mathbf{w}_{t-1} - \lambda \mathbf{m}_t.
        \end{cases}
    \end{equation}
    where $\mu\in[0,1)$ is the momentum hyperparameter, $S := |B_t|$ is the minibatch size,
    and the set $B_t=\{i_1,...i_S\}$ are $S$ i.i.d. random integers sampled uniformly from $[1,N]$. 
\end{definition}

\begin{table*}[t!]
\centering

    \caption{\small Summary of related works on the noise and stationary distribution of SGD. This work fills the gap of the lack of theoretical results for the actual SGD dynamics, which is \textit{discrete-time} and with \textit{minibatch noise}.\vspace{-1em}}\label{tab: related works}
{\scriptsize
\resizebox{\textwidth}{!}{
    \begin{tabular}{c|ccc}
    \hline
    Setting & Artificial Noise & Hessian Approximation Noise &  \textbf{\textit{Minibatch Noise}}  \\
    \hline
    Continuous-time   & \cite{sato2014approximation, Welling2011} & \cite{jastrzebski2018three, pmlr-v97-zhu19e}   & \cite{blanc2020implicit, mori2021logarithmic}\\
     & \cite{Mandt2017, meng2020dynamic} &\cite{wu2020noisy,Xie2020} & \\
    \hline
    \textbf{\textit{Discrete-time}}     & \cite{yaida2018fluctuation, gitman2019understanding}  & \cite{liu2021noise} & \textbf{This Work}\\
    & \cite{liu2021noise} & \\
    \hline
    \end{tabular}}}
    \vspace{-1em}
\end{table*}

One can decompose the gradient into a deterministic plus a stochastic term. Note that $\mathbb{E}_{\rm B}[\hat{\mathbf{g}}_t] = \nabla L $ is equal to the gradient for the GD algorithm. We use $\mathbb{E}_{\rm B}(\cdot)$ to denote the expectation over batches, and use $\mathbb{E}_{\mathbf{w}}(\cdot)$ to denote the expectation over the stationary distribution of the model parameters. Therefore, we can write $\hat{\mathbf{g}}_t = \mathbb{E}_{\rm B}[\hat{\mathbf{g}}_t]  + \eta_t,$ where $\eta_t:= \frac{1}{S}\sum_{i\in B_t} \nabla  \ell(x_i, y_i, \mathbf{w}_{t-1}) - \mathbb{E}_{\rm B}[\hat{\mathbf{g}}_t]$ is the noise term; the noise covariance is $C(\mathbf{w}_t):= \text{cov}(\eta_t, \eta_t)$. Of central importance to us is the averaged asymptotic noise covariance $C:=\lim_{t\to \infty} \mathbb{E}_{\mathbf{w}_t}[C(\mathbf{w}_t)]$. Also, we consider the asymptotic model fluctuation $\Sigma:= \lim_{t\to\infty} \text{cov}(\mathbf{w}_t, \mathbf{w}_t)$. $\Sigma$ gives the strength and shape of the fluctuation of $\mathbf{w}$ around a local minimum and is another quantity of central importance to this work. Throughout this work, $C$ is called the ``noise" and $\Sigma$ the ``fluctuation".

\vspace{-3mm}
\section{Related Works}\label{sec: related works}
\vspace{-2mm}

\textbf{Noise and Fluctuation in SGD}. Deep learning models are trained with SGD and its variants. To understand the parameter distribution in deep learning, one needs to understand the stationary distribution of SGD \citep{Mandt2017}. \citet{sato2014approximation} describes the stationary distribution of stochastic gradient Langevin dynamics using discrete-time Fokker-Planck equation. \citet{yaida2018fluctuation} connects the covariance of parameter $\Sigma$ to that of the noise $C$ through the fluctuation-dissipation theorem. 
When $\Sigma$ is known, one may obtain by Laplace approximation the stationary distribution of the model parameter around a local minimum $w^*$ as  $\mathcal{N}(w^*, \Sigma)$. Therefore, knowing $\Sigma$ can be of great practical use. For example, it has been used to estimate the local minimum escape efficiency \citep{pmlr-v97-zhu19e, liu2021noise} and argue that SGD prefers a flatter minimum; it can also be used to assess parameter uncertainty and prediction uncertainty when a Bayesian prior is specified \citep{Mandt2017, gal2016dropout, pmlr-v108-pearce20a}. Empirically, both the fluctuation and the noise are known to crucially affect the generalization of a deep neural network. \cite{wu2020noisy} shows that the strength and shape of the $\Sigma$ due to the minibatch noise lead to better generalization of neural networks in comparison to an artificially constructed noise.

\textbf{Hessian Approximation of the Minibatch Noise.} However, it is not yet known what form $C$ and $\Sigma$ actually take for SGD in a realistic learning setting. Early attempts assume isotropic noise in the continuous-time limit \citep{sato2014approximation, Mandt2017}. In this setting, the noise is an isotropic Gaussian with $C\sim I_D$, and $\Sigma$ is known to be proportional to the inverse Hessian $H^{-1}$.
More recently, the importance of noise structure was realized  \citep{Hoffer2017,jastrzebski2018three,pmlr-v97-zhu19e,haochen2020shape}. ``Hessian approximation", which assumes $C\approx c_0H$ for some unknown constant $c_0$, has often been adopted for understanding SGD (see Table~\ref{tab: related works}); this assumption is often motivated by the fact that $C=J_w \approx H$, where $J_w$ is the Fisher information matrix (FIM) \citep{pmlr-v97-zhu19e}; the fluctuation can be solved to be isotropic: $\Sigma \sim I_D$. However, it is not known under what conditions the Hessian approximation is valid, while previous works have argued that it can be very inaccurate \citep{martens2014insights, liu2021noise, thomas2020interplay, kunstner2019limitations}. However, \cite{martens2014insights} and \cite{kunstner2019limitations} only focuses on the natural gradient descent (NGD) setting; \cite{thomas2020interplay} is closest to ours, but it does not apply to the case with momentum, a matrix learning rate, or regularization.

\textbf{Discrete-time SGD with a Large Learning Rate}. Recently, it has been realized that networks trained at a large learning rate have a dramatically better performance than networks trained with a vanishing learning rate (lazy training) \citep{chizat2018note}. \cite{Lewkowycz2020} shows that there is a qualitative difference between the lazy training regime and the large learning rate regime; the performance features two plateaus in testing accuracy in the two regimes, with the large learning rate regime performing much better. However, the theory regarding discrete-time SGD at a large learning rate is almost non-existent, and it is also not known what $\Sigma$ may be when the learning rate is non-vanishing. Our work also sheds light on the behavior of SGD at a large learning rate. Some other works also consider discrete-time SGD in a similar setting \citep{fontaine2021convergence, dieuleveut2017bridging, toulis2016asymptotic}, but the focus is not on deriving analytical formulae or does not deal with the stationary distribution. 
\vspace{-3mm}
\section{SGD Noise and Fluctuation in Linear Regression}
\label{sec:theory}
\vspace{-2mm}

This section derives the shape and strength of SGD noise and fluctuation for linear regression; concurrent to our work, \cite{kunin2021rethinking} also studies the same problem but with continuous-time approximation; our result is thus more general. To emphasize the message, we discuss the label noise case in more detail. The other situations also deserve detailed analysis; we delay such discussion to the appendix due to space constraints. \textbf{Notation}: $S$ denotes the minibatch size. $\mathbf{w}\in \mathbb{R}^D$ is the model parameter viewed in a vectorized form; $\lambda \in \mathbb{R}_+$ denotes a scalar learning rate; when the learning rate takes the form of a preconditioning matrix, we use $\Lambda \in \mathbb{R}^{D\times D}$. $A\in \mathbb{R}^{D\times D}$ denotes the covariance matrix of the input data. When a matrix $X$ is positive semi-definite, we write $X\geq 0$; throughout, we require $\Lambda\geq 0$. $\gamma \in \mathbb{R}$ denotes the weight decay hyperparameter; when the weight decay hyperparameter is a matrix, we write $\Gamma \in \mathbb{R}^{D\times D}$. $\mu$ is the momentum hyperparameter in SGD. For two matrices $X, Y$, the commutator is defined as $[X, Y]:=XY-YX$. Other notations are introduced in the context.\footnote{We use the word \textit{global minimum} to refer to the global minimum of the loss function, i.e., where $L=0$ and a \textit{local minimum} refers to a minimum that has a non-negative loss, i.e., $L\geq 0$.} The results of this section are numerically verified in Appendix~\ref{app: extra exp}.

\vspace{-3mm}
\subsection{Key Previous Results}
\vspace{-2mm}

When $N\gg S$, the following proposition is well-known and gives the exact noise due to minibatch sampling. See Appendix~\ref{app: der of C} for a derivation.
\begin{proposition}\label{prop: noise_cov}
The noise covariance of SGD as defined in Definition~\ref{def: with replacement} is\vspace{-2mm}
\begin{equation}\label{eq:noise_cov}
    C(\mathbf{w})=\frac{1}{SN}\sum_{i=1}^{N}\nabla \ell_i(\mathbf{w}) \nabla \ell_i(\mathbf{w})^{\rm T}-\frac{1}{S}\nabla L(\mathbf{w})\nabla L(\mathbf{w})^{\rm T},\vspace{-2mm}
\end{equation}
where the notations $\ell_i(\mathbf{w}):= l(x_i, y_i, \mathbf{w})$ and $L(\mathbf{w}):=L(\{x_i, y_i\}_{i=1}^{N}, \mathbf{w})$ are used.
\end{proposition}
This gradient covariance matrix $C$ is crucial to understand the minibatch noise. 
The standard literature often assumes $C(\mathbf{w})\approx H(\mathbf{w})$; however, the following well-known proposition shows that this approximation can easily break down.
\begin{proposition}\label{prop: no noise condition}
    Let $\mathbf{w}_*$ be the solution such that $L(\mathbf{w}_*)=0$, then $C(\mathbf{w}_*)=0$.\vspace{-1mm}
\end{proposition}
\textit{Proof}. Because $\ell_i$ is non-negative for all $i$, $L(\mathbf{w}_*)=0$ implies that $\ell_i(\mathbf{w}_*)=0$. The differentiability in turn implies that each $\nabla\ell_i(\mathbf{w}_*)=0$; therefore, $C=0$. $\square$

This proposition implies that there is no noise if our model can achieve zero training loss (which is achievable for an overparametrized model).
This already suggests that the Hessian approximation $C\sim H$ is wrong since the Hessian is unlikely to vanish in any minimum. The fact that the noise strength vanishes at $L=0$ suggests that the SGD noise might at least be proportional to $L(\mathbf{w})$, which we will show to be true for many cases. The following theorem relates $C$ and $\Sigma$ of the discrete-time SGD algorithm with momentum for a matrix learning rate.

\begin{theorem}\label{theo: liu2021}
    \citep{liu2021noise} Consider running SGD on a quadratic loss function with Hessian $H$, learning rate matrix $\Lambda$, momentum $\mu$. Assuming ergodicity, then
    \begin{equation}
    (1-\mu) (\Lambda H\Sigma + \Sigma H \Lambda) - \frac{1+\mu^2}{1-\mu^2}\Lambda H\Sigma H \Lambda+ \frac{\mu}{1-\mu^2}(\Lambda H\Lambda H\Sigma +\Sigma H\Lambda H\Lambda)  = \Lambda C\Lambda. \label{eq: preconditioning matrix eq}
\end{equation}
\end{theorem}
 Propostion~\ref{prop: noise_cov} and Theorem~\ref{theo: liu2021} allow one to solve $C$ and $\Sigma$. Equation~\eqref{eq: preconditioning matrix eq} can be seen as a general form of the Lyapunov equation \citep{lyapunov1992general} and is hard to solve in general \citep{hammarling1982numerical, ye1998stability, simoncini2016computational}. Solving this analytical equation in settings of machine learning relevance is one of the main technical contributions of this work. 


\vspace{0mm}
\subsection{Random Noise in the Label}\label{sec: label noise}
\vspace{-1mm}

We first consider the case when the labels contain noise. The loss function takes the form
\begin{equation}
    L(\mathbf{w}) = \frac{1}{2N}\sum_{i=1}^N(\mathbf{w}^{\rm T} x_i - y_i)^2, \label{eq: loss fucntion label}
\end{equation}
where $x_i\in \mathbb{R}^D$ are drawn from a zero-mean Gaussian distribution with feature covariance $A:= \mathbb{E}_{\rm B}[xx^{\rm T}]$, and $y_i=\mathbf{u}^{\rm{T}}x_i + \epsilon_i$, for some fixed $\mathbf{u}$ and $\epsilon_i \in \mathbb{R}$ is drawn from a distribution with zero mean and finite second momentum $\sigma^2$. We redefine $\mathbf{w} -\mathbf{u} \to \mathbf{w} $ and let $N \to \infty$ with $D$ held fixed. The following lemma finds $C$ as a function of $\Sigma$.

\begin{lemma}\label{prop: label noise covariance}$($Covariance matrix for SGD noise in the label$)$ Let $N\to\infty$ and the model be updated according to Eq.~\eqref{eq: with replacement} with loss function in Eq.~\eqref{eq: loss fucntion label}. Then,\vspace{-1mm}
\begin{equation}
    C   = \frac{1}{S}(A\Sigma A + {\rm Tr}[A\Sigma]A+ \sigma^2 A).\vspace{-1mm}\label{eq: label noise covariance}
\end{equation}
\end{lemma}
The model fluctuation can be obtained using this lemma.

\begin{theorem}\label{thm: Sigma label}$($Fluctuation of model parameters with random noise in the label$)$ Let the assumptions be the same as in Lemma~\ref{prop: label noise covariance} and $[\Lambda,A]=0$. Then,\vspace{-1mm}
\begin{equation}\label{eq: solution}
     \Sigma  = \frac{\sigma^2}{S}\left( 1 + \frac{ \kappa_{\mu}}{S}\right)  \Lambda G_{\mu}^{-1},\vspace{-1mm}
\end{equation}
where $\kappa_{\mu}:=\frac{ {\rm Tr}[\Lambda AG_{\mu}^{-1}]}{1 -   \frac{1}{S}{\rm Tr} [\Lambda AG_{\mu}^{-1}]}$ with $G_{\mu}:=2(1-\mu)I_D-\left(\frac{1-\mu}{1+\mu}+\frac{1}{S}\right)\Lambda A$.
\end{theorem}

\begin{remark}
This result is numerically validated in Appendix~\ref{app: extra exp}. The subscript $\mu$ refers to momentum. To obtain results for vanilla SGD, one can set $\mu=0$, which has the effect of reducing $G_{\mu}\to G=2I_D - \left(1+\frac{1}{S}\right)\Lambda A$. From now on, we focus on the case when $\mu=0$ for notational simplicity, but we note that the results for momentum can be likewise studied. The assumption $[\Lambda, A]=0$ is not too strong because this condition holds for a scalar learning rate and common second-order methods such as Newton's method.
\end{remark}
If $\sigma^2=0$, then $\Sigma=0$. This means that when there is no label noise, the model parameter has a vanishing stationary fluctuation, which corroborates Proposition~\ref{prop: no noise condition}. When a scalar learning rate $\lambda \ll 1$ and $1 \ll S$, we have\vspace{0mm}
\begin{equation}\label{eq: cts time approximation}
    \Sigma \approx \frac{\lambda\sigma^2}{2 S} I_D,\vspace{-0mm}
\end{equation}
which is the result one would expect from the continuous-time theory with the Hessian approximation \citep{liu2021noise, Xie2020, pmlr-v97-zhu19e}, except for a correction factor of $\sigma^2$. Therefore, a Hessian approximation fails to account for the randomness in the data of strength $\sigma^2$. We provide a systematic and detailed comparison with the Hessian approximation in  Table~\ref{tab:summary} of Appendix~\ref{app: comparison}. 

Moreover, it is worth comparing the exact result in Theorem~\ref{thm: Sigma label} with Eq.~\eqref{eq: cts time approximation} in the regime of non-vanishing learning rate and small batch size. One notices two differences: (1) an anisotropic enhancement, appearing in the matrix $G_{\mu}$ and taking the form $-\lambda(1+1/S)A$; compared with the result in \citet{liu2021noise}, this term is due to the compound effect of using a large learning rate and a small batchsize; (2) an isotropic enhancement term $\kappa$, which causes the overall magnitude of fluctuations to increase; this term does not appear in the previous works that are based on the Hessian approximation and is due to the minibatch sampling process alone. As the numerical example in Appendix~\ref{app: extra exp} shows, at large batch size, the discrete-time nature of SGD is the leading source of fluctuation; at small batch size, the isotropic enhancement becomes the dominant source of fluctuation. Therefore, the minibatch sampling process causes two different kinds of enhancement to the fluctuation, potentially increasing the exploration power of SGD at initialization but reducing the convergence speed. 

Now, combining Theorem~\ref{thm: Sigma label} and Lemma~\ref{prop: label noise covariance}, one can obtain an explicit form of the noise covariance.
\begin{theorem}\label{prop: explicit C label}
The noise covariance matrix of minibatch SGD with random noise in the label is 
\begin{equation}\label{eq: explicit C label}
    C = \frac{\sigma^2}{\color{orange} S} {\color{orange} A} +\frac{\sigma^2}{S^2}\left(1+\frac{\kappa_{\mu}}{S}\right)\left(\Lambda AG_{\mu}^{-1}+{\rm Tr} [\Lambda AG_{\mu}^{-1}] I_D\right)A.
\end{equation}
\end{theorem}
By definition, $C=J$ is the FIM. The Hessian approximation, in sharp contrast, can only account for the term in {\color{orange} orange}. A significant modification containing both anisotropic and isotropic (up to Hessian) is required to fully understand SGD noise, even in this simple example. Additionally, comparing this result with the training loss \eqref{eq: trainingloss label}, one can find that the noise covariance contains one term that is proportional to the training loss. In fact, we will derive in Sec.~\ref{sec:general} that containing a term proportional to training loss is a general feature of the SGD noise. We also study the case when the input is contaminated with noise. Interestingly, the result is the same with the label noise case with $\sigma^2$ replaced by a more complicated term of the form ${\rm Tr}[AK^{-1}BU]$. We thus omit this part from the main text. A detailed discussion can be found in Appendix~\ref{sec: input noise}. In the next section, we study the effect of regularization on SGD noise and fluctuation.

\vspace{-2mm}
\subsection{Learning with Regularization}\label{sec: regularization}
\vspace{-1mm}
Now, we show that regularization also causes a unique SGD noise. The loss function for $\Gamma-L_2$ regularized linear regression is \vspace{-1mm}
\begin{align}
    L_\Gamma(\mathbf{w}) &=\frac{1}{2N}\sum_{i=1}^{N}\left[(\mathbf{w}-\mathbf{u})^{\rm T}x_i\right]^2 + \frac{1}{2} \mathbf{w}^{\rm T} \Gamma \mathbf{w} = \frac{1}{2}(\mathbf{w}-\mathbf{u})^{\rm T}A(\mathbf{w}-\mathbf{u})+ \frac{1}{2} \mathbf{w}^{\rm T} \Gamma \mathbf{w},\label{eq: loss function regular}
\end{align}
where $\Gamma$ is a symmetric matrix; conventionally, one set $\Gamma=\gamma I_D$ with a scalar $\gamma>0$. For conciseness, we assume that there is no noise in the label, namely $y_i=\mathbf{u}^{\rm T}x_i$ with a constant vector $\mathbf{u}$. One important quantity in this case will be $\mathbf{u} \mathbf{u}^{\rm T}  :=  U$. The noise for this form of regularization can be calculated but takes a complicated form.

\begin{proposition}\label{prop: C of regular}$($Noise covariance matrix for learning with L$_2$ regularization$)$ 
Let the algorithm be updated according to Eq.~\eqref{eq: with replacement} on loss function~\eqref{eq: loss function regular} with $N\to \infty$ and $[A,\Gamma]=0$. Then,
\vspace{-1mm}
\begin{equation}
    C=\frac{1}{S}\left(A\Sigma A+{\rm Tr}[A\Sigma]A+{\rm Tr}[\Gamma'^{\rm T} A\Gamma' U]A + \Gamma A' U A' \Gamma\right),\vspace{-1mm}
\end{equation}
where $A':=K^{-1}A$, $\Gamma':=K^{-1}\Gamma$ with $K:=A+\Gamma$. 
\end{proposition}

Notice that the last term $\Gamma A' U A' \Gamma$ in $C$ is unique to the regularization-based noise: it is rank-1 because $U$ is rank-1. This term is due to the mismatch between the regularization and the minimum of the original loss. Also, note that the term ${\rm Tr}[A\Sigma]$ is proportional to the training loss. Define the test loss to be $L_{\rm test}:= \lim_{t\to \infty} \mathbb{E}_{\mathbf{w}_t}[\frac{1}{2}(\mathbf{w}_t-\mathbf{u})^{\rm T}A(\mathbf{w}_t-\mathbf{u})]$, we can prove the following theorem. We will show that one intriguing feature of discrete-time SGD is that the weight decay can be negative.

\begin{theorem}\label{thm: errors of regular}$($Test loss and model fluctuation for L$_2$ regularization$)$ 
Let the assumptions be the same as in Proposition~\ref{prop: C of regular}. Then 
\begin{equation}
    L_{\rm test}=\frac{\lambda}{2S}\left({\rm Tr}[AK^{-2}\Gamma^2 U ]\kappa+r\right)+\frac{1}{2}{\rm Tr}[AK^{-2}\Gamma^2U],\label{eq: regular test loss}
\end{equation}%
where $\kappa:=\frac{{\rm Tr}[A^2 K^{-1}G^{-1}]}{1-\frac{\lambda}{S}{\rm Tr}[A^2 K^{-1}G^{-1}]}$, $r:=\frac{{\rm Tr}[A^3 K^{-3} \Gamma^2 G^{-1}U]}{1-\frac{\lambda}{S}{\rm Tr}[A^2 K^{-1} G^{-1}]}$, with $G:=2I_D-\lambda\left(K+\frac{1}{S}K^{-1}A^2\right)$. Moreover, let $[\Gamma, U] =0$, then \vspace{-1mm}
\begin{align}
    \Sigma=&\frac{\lambda}{S}{\rm Tr}[AK^{-2}\Gamma^2U]\left(1+\frac{\lambda \kappa}{S}\right)AK^{-1}G^{-1}+\frac{\lambda}{S}\left(A^2 K^{-2}\Gamma^2 U+\frac{\lambda r}{S}A \right)K^{-1}G^{-1}.\label{eq: regularization solution}
\end{align}
\end{theorem}

This result is numerically validated in Appendix~\ref{app: extra exp}. The test loss \eqref{eq: regular test loss} has an interesting consequence. One can show that there exist situations where the optimal $\Gamma$ is \textit{negative}.\footnote{Some readers might argue that discussing test loss is meaningless when $N\to \infty$; however, this criticism does not apply because the size of the training set is not the only factor that affects generalization. In fact, this section's crucial message is that using a large learning rate affects the generalization by implicitly regularizing the model and, if one over-regularizes, one needs to offset this effect.} When discussing the test loss, we make the convention that if $\mathbf{w}_t$ diverges, then $L_{\rm test} =\infty$.
\begin{corollary}\label{cor: negative gamma}
Let $\gamma^*= \arg\min_\gamma L_{\rm test}$. There exist $a$, $\lambda$ and $S$ such that $\gamma^*<0$. 
\end{corollary}

The proof shows that when the learning rate is sufficiently large, only negative weight decay is allowed. This agrees with the argument in \cite{liu2021noise} that discrete-time SGD introduces an implicit $L_2$ regularization that favors small norm solutions. A too-large learning rate requires a negative weight decay because a large learning rate already over-regularizes the model and one needs to introduce an explicit negative weight decay to offset this over-regularization effect of SGD. This is a piece of direct evidence that using a large learning rate can help regularize the models. It has been hypothesized that the dynamics of SGD implicitly regularizes neural networks such that the training favors simpler solutions \citep{kalimeris2019sgd}. Our result suggests one new mechanism for such a regularization.

\vspace{-3mm}
\section{Noise Structure for Generic Settings}
\label{sec:general}
\vspace{-2mm}

The results in the previous sections suggest that (1) the SGD noises differ for different kinds of situations, and (2) SGD noise contains a term proportional to the training loss in general. These two facts motivate us to derive the noise covariance differently for different kinds of minima. Let $f(\mathbf{w},x)$ denote the output of the model for a given input $x\in\mathbb{R}^D$. 
Here, we consider a more general case; $f(\mathbf{w},x)$ may be any differentiable function, e.g., a non-linear deep neural network.
The number of parameters in the model is denoted by $P$, and hence $\mathbf{w}\in\mathbb{R}^P$.
For a training dataset $\{x_i,y_i\}_{i=1,2,\dots,N}$, the loss function with a $L_2$ regularization is given by\vspace{-1mm}
\begin{equation}
    L_\Gamma(\mathbf{w})=L_0(\mathbf{w})+\frac{1}{2}\mathbf{w}^\mathrm{T}\Gamma\mathbf{w},\label{eq: general loss}
\end{equation}
where $L_0(\mathbf{w})=\frac{1}{N}\sum_{i=1}^N\ell(f(\mathbf{w},x_i),y_i)$ is the loss function without regularization, and $H_0$ is the Hessian of $L_0$.
We focus on the MSE loss $\ell(f(\mathbf{w},x_i),y_i)=[f(\mathbf{w},x_i)-y_i]^2/2$.
Our result crucially relies on the following two assumptions, which relate to the conditions of different kinds of local minima.
 \begin{assumption}\label{assum 1}
    (Fluctuation decays with batch size) $\Sigma$ is proportional to $S^{-1}$, i.e. $\Sigma = O(S^{-1})$.
 \end{assumption}
This is justified by the results in all the related works \citep{liu2021noise, Xie2020, meng2020dynamic, mori2021logarithmic}, where $\Sigma$ is found to be $O(S^{-1})$.

 \begin{assumption}\label{assum 2}
    (Weak homogeneity) $|L - \ell_i|$ is small; in particular, it is of order $o(L)$. 
 \end{assumption}

This assumption amounts to assuming that the current training loss $L$ reflects the actual level of approximation for each data point well. In fact, since $L\geq 0$, one can easily show that $|L - \ell_i|=O(L)$. Here, we require a slightly stronger condition for a more clean expression, when $|L - \ell_i|=O(L)$ we can still get a similar expression but with some constant that hinders the clarity. Relaxing this condition can be an important and interesting future work. The above two conditions allow us to state our general theorem formally.

\begin{theorem}\label{theo: general formula}
    Let the training loss be $L_\Gamma = L_0 + \frac{1}{2}\mathbf{w}^{\rm T}\Gamma \mathbf{w}$ and the models be optimized with SGD in the neighborhood of a local minimum $\mathbf{w}^*$. Then,
    \begin{equation}
        C(\mathbf{w}) =  \frac{2L_0(\mathbf{w})}{S}H_0(\mathbf{w})-\frac{1}{S}\nabla L_{\Gamma}(\mathbf{w})\nabla L_{\Gamma}(\mathbf{w})^\mathrm{T} + o(L_0).
        \label{eq:cov_Hessian}%
    \end{equation}
\end{theorem}
    
The noise takes different forms for different kinds of local minima.

\begin{corollary}\label{corr: different minima}
    Omitting the terms of order $o(L_0)$, when $\Gamma \neq 0$, 
    \begin{equation}
        C = \frac{2L_0(\mathbf{w}^*)}{S}H_0(\mathbf{w}^*)-\frac{1}{S}\Gamma\mathbf{w}^*\mathbf{w}^{*\mathrm{T}}\Gamma + O(S^{-2}) + O(|\mathbf{w} - \mathbf{w}^*|^2). 
    \end{equation}
    When $\Gamma = 0$ and $L_0(\mathbf{w}^*) \neq 0$,
    \begin{equation}
        C = \frac{2L_0(\mathbf{w}^*)}{S}H_0(\mathbf{w}^*) + O(S^{-2}) + O(|\mathbf{w} - \mathbf{w}^*|^2).
    \end{equation}
    When $\Gamma = 0$ and $L_0(\mathbf{w}^*) = 0$,
    \begin{equation}
        C = \frac{1}{S} \left({\rm Tr}[H_0(\mathbf{w}^*) \Sigma] I_D - H_0(\mathbf{w}^*) \Sigma\right) H_0(\mathbf{w}^*) + O(S^{-2}) + O(|\mathbf{w} - \mathbf{w}^*|^2).
    \end{equation}
\end{corollary} 

\begin{remark}
    Assumption~\ref{assum 2} can be replaced by a weaker but more technical assumption called the ``decoupling assumption", which has been used in recent works to derive the continuous-time distribution of SGD \citep{mori2021logarithmic, wojtowytsch2021stochasticb}. The Hessian approximation was invoked in most of the literature without considering the conditions of its applicability \citep{jastrzebski2018three, pmlr-v97-zhu19e, liu2021noise, wu2020noisy, Xie2020}. Our result does provide such conditions for applicability. As indicated by the two assumptions, this theorem is applicable when the batch size is not too small and when the local minimum has a loss close to $0$. The reason for the failure of the Hessian approximation is that, while the FIM is equal to the expected Hessian $J=\mathbb{E}[H]$, there is no reason to expect the expected Hessian to be close to the actual Hessian of the minimum. 
\end{remark}

The proof is given in Appendix~\ref{app: general}. Two crucial messages this corollary delivers are (1) the SGD noise is different in strength and shape in different kinds of local minima and that they need to be analyzed differently; (2) the SGD noise contains a term that is proportional to the training loss $L_0$ in general. Recently, it has been experimentally demonstrated that the SGD noise is indeed proportional to the training loss in realistic deep neural network settings, both when the loss function is MSE and cross-entropy \citep{mori2021logarithmic}; our result offers a theoretical justification. The previous works all treat all the minima as if the noise is similar \citep{jastrzebski2018three, pmlr-v97-zhu19e, liu2021noise, wu2020noisy, Xie2020}, which can lead to inaccurate or even incorrect understanding. For example, Theorem 3.2 in \cite{Xie2020} predicts a high escape probability from a sharp local or global minimum. However, this is incorrect because a model at a global minimum has zero probability of escaping due to a vanishing gradient. In contrast, the escape rate results derived in \cite{mori2021logarithmic} correctly differentiate the local and global minima. We also note that these general formulae are consistent with the exact solutions we obtained in the previous section than the Hessian approximation. For example, the dependence of the noise strength on the training loss in Theorem~\ref{thm: Sigma label}, and the rank-$1$ noise of regularization are all reflected in these formulae. In contrast, the simple Hessian approximation misses these crucial distinctions. 
Lastly, combining Theorem~\ref{theo: general formula} with Theorem~\ref{theo: liu2021}, one can also find the fluctuation.
\begin{corollary}\label{cor: general Sigma}
Let the noise be as in Theorem~\ref{theo: general formula}, and omit the terms of order $O(S^{-2})$ and $O(|\mathbf{w}- \mathbf{w}^*|^2)$. Then, when $\Gamma \neq 0 $ and when $\Lambda$, $H_0(\mathbf{w}^*)$ and $\Gamma$ commute with each other, $P_{r'}\Sigma=\frac{1}{S}\frac{\Lambda}{1-\mu}(2L_0H_0-\Gamma\mathbf{w}^*\mathbf{w}^{*{\rm T}}\Gamma)(H_0+\Gamma)^{+}\left[2I_D - \frac{\Lambda}{1+\mu}(H_0+\Gamma)\right]^{-1}$. When $\Gamma = 0$ and $L_0(\mathbf{w}^*) \neq 0$, $P_r \Sigma = \frac{2L_0}{S(1-\mu)}P_r\Lambda\left(2I_D - \frac{\Lambda}{1+\mu}H_0\right)^{-1}$. When $\Gamma = 0$ and $L_0(\mathbf{w}^*) = 0$, $P_r\Sigma = 0$. Here the superscript $+$ is the Moore-Penrose pseudo inverse, $P_r := {\rm diag}(1,\dots,1,0,\dots,0)$ is the projection operator with $r$ non-zero entries, $r\le D$ is the rank of the Hessian $H_0$, and $r'\le D$ is the rank of $H_0 + \Gamma$. For the null space $H_0$, $\Sigma$ can be arbitrary.
\end{corollary}

\vspace{-3mm}
\section{Applications}\label{sec: applications}
\vspace{-2mm}
One major advantage of analytical solutions is that they can be applied in a simple ``plug-in" manner by the practitioners or theorists to analyze new problems they encounter. In this section, we briefly outline a few examples where the proposed theories can be relevant. 

\vspace{-2mm}
\subsection{High-Dimensional Regression}\label{sec: high dimension linear regression}
\vspace{-1mm}
We first apply our result to the high-dimensional regression problem and show how over-and-underparametrization might play a role in determining the minibatch noise. Here, we take $N,D\to\infty$ with the ratio $\alpha:=N/D$ held fixed. The loss function is $L(\mathbf{w})=\frac{1}{2N}\sum_{i=1}^N\left(\mathbf{w}^\mathrm{T}x_i-y_i\right)^2$. As in the standard literature \citep{hastie2019surprises}, we assume the existence of label noise: $y_i = \mathbf{u}^{\rm T}x_i + \epsilon_i$, with $\text{Var}[\epsilon_i] =\sigma^2$. A key difference between our setting and the standard high-dimensional setting is that, in the standard setting \citep{hastie2019surprises}, one uses the GD algorithm with vanishing learning rate $\lambda$ instead of the minibatch SGD algorithm with a non-vanishing learning rate. Tackling the high-dimensional regression problem with non-vanishing $\lambda$ and a minibatch noise is another main technical contribution of this work. In this setting, we can obtain the following result on the noise covariance matrix.
\begin{proposition}\label{prop: high D C}
Let $\hat{A}= \frac{1}{N}\sum_i^{N} x_i x^{\rm T}_i$ and suppose assumptions~\ref{assum 1} and \ref{assum 2} hold. With fixed $S$, $\lambda$, then
    $C = \frac{1}{S}\left(\mathrm{Tr}[\hat{A}\Sigma]I_D-\hat{A}\Sigma\right)\hat{A}+ \max\left\{0,  \frac{\sigma^2}{S}\left(1-\frac{1}{\alpha}\right)\right\}\hat{A}$.
\end{proposition}
\vspace{-2mm}
We note that this proposition follows from Theorem~\ref{theo: general formula}, showing an important theoretical application of our general theory. An interesting observation is that one $\Sigma$-independent term proportional to $\sigma^2$ emerges in the underparametrized regime ($\alpha > 1$). However, for the overparametrized regime, the noise is completely dependent on $\Sigma$, which is a sign that the stationary solution has no fluctuation. This shows that the degree of underparametrization also plays a distinctive role in the fluctuation. In fact, one can prove the following theorem, which is verified in Appendix~\ref{app sec: high dimension exp}.
\begin{theorem}\label{thm: trace of high D}
    When a stationary solution exists for $\mathbf{w}$, we have ${\rm Tr}[\hat{A}\Sigma] = \max\left\{0,  \frac{\lambda\sigma^2}{S}\left(1-\frac{1}{\alpha}\right)\hat{\kappa} \right\}$, where 
    $\hat{\kappa}:=\frac{{\rm Tr}[\hat{G}^{-1}\hat{A}]}{1-\frac{\lambda}{S}{\rm Tr}[\hat{G}^{-1}\hat{A}]}$ with $\hat{G}:=2I_D-\lambda\left(1-\frac{1}{S}\right)\hat{A}$.
\end{theorem}



\begin{figure}[tb!]
\vspace{-1em}
    \centering
    \includegraphics[width=0.24\linewidth]{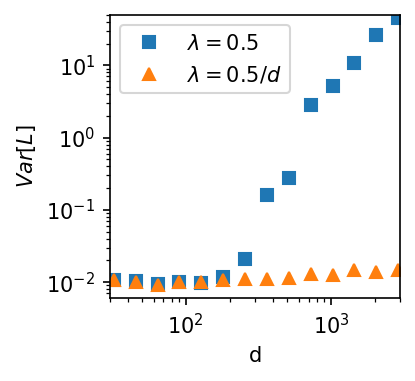}
    \includegraphics[width=0.24\linewidth]{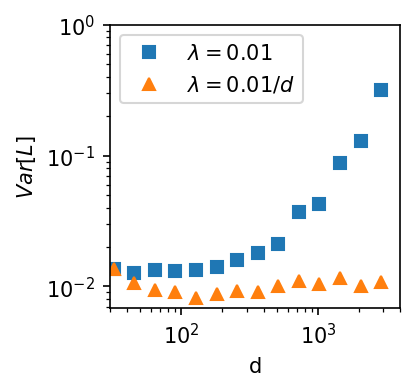}
    \includegraphics[width=0.26\linewidth]{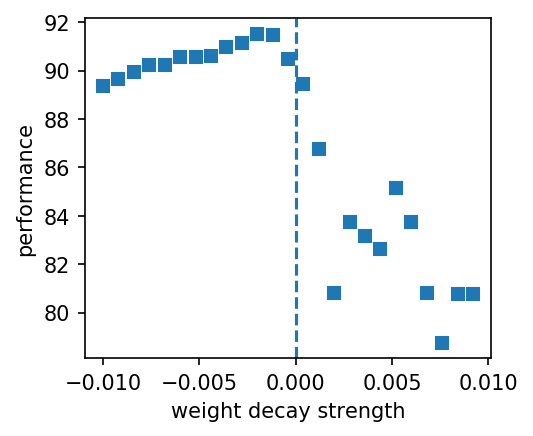}
    \vspace{-1em}
\caption{\small Realistic learning settings with neural networks and logistic regression. \textbf{Left}: Variance of training loss of a neural network with width $d$ and tanh activation on the MNIST dataset. We see that the variance explodes after $d\geq 200$. In contrast, rescaling the learning rate by $1/d$ results in a constant noise level in training. This suggests that the stability condition we derived for high-dimension regression is also useful for understanding deep learning. \textbf{Middle}: Stability of Adam with the same setting. Adam also experiences a similar stability problem when the model width increases. \textbf{Right}: Logistic regression on MNIST trained with SGD; with $\lambda=1.5$, $S=32$. We see that the optimal performance is also achieved at negative weight decay strength $\gamma$, suggesting that a large learning rate can indeed introduce effective regularization.}
\label{fig: experiments 1}
\vspace{-1em}
\end{figure}

\vspace{-3mm}
\subsection{Implication for Neural Network Training}
\vspace{-2mm}
It is commonly believed that the high-dimensional linear regression problem can be a minimal model for deep learning. Taking this stance, Theorem~\ref{thm: trace of high D} suggests a technique for training neural networks. For SGD to converge, a positive semi-definite $\Sigma$ must exist; however, $\Sigma\geq 0$ if and only if $\hat{\kappa}\geq 0$. From $\hat{\kappa}>0$, we have $\sum_{i=1}^{D}\frac{1}{2/\lambda a_i - 1+1/S}<S$, where $a_i$ are the eigenvalues of $\hat{A}$. This means that each summand should have the order of $D/S$. Thus the upper bound of $\lambda$ should have the order of $2S/aD$, where $a$ is the typical value of $a_i$'s. One implication of the dependence on the dimension is that the stability of a neural network trained with SGD may strongly depend on its width $d$, and one may rescale the learning rate according to the width to stabilize neural network training. See Figure~\ref{fig: experiments 1}-Left and Middle. We train a two-layer tanh neural network on MNIST and plot the variance of its training loss in the first epoch with fixed $\lambda=0.5$. We see that, when $d\geq 200$, the training starts to destabilize, and the training loss begins to fluctuate dramatically. When rescaling the learning rate by $1/d$, we see that the variance of the training loss is successfully kept roughly constant across all $d$. This suggests a training technique worth being explored by practitioners in the field. In Figure~\ref{fig: experiments 1}-Middle, we also use Adam for training the same network and find a similar stabilizing trick to work for Adam.

\vspace{-3mm}
\subsection{A Natural Learning Example with Negative Weight Decay}
\vspace{-2mm}

Sec.~\ref{sec: regularization} shows that a too-large learning rate introduces an effective $L_2$ regularization that can be corrected by setting the weight decay to be negative. This effect can be observed in more realistic learning settings. We train a logistic regressor on the MNIST dataset with a large learning rate (of order $O(1)$). Figure~\ref{fig: experiments 1}-Right confirms that, at a large learning rate, the optimal weight decay can indeed be negative. This agrees with our argument that using a large learning rate can effectively regularize the training.

\vspace{-2mm}
\subsection{Second-order Methods}\label{sec: second order}
\vspace{-2mm}

Understanding stochastic second-order methods (including the adaptive gradient methods) is also important for deep learning \citep{agarwal2017second, zhang2021distributional, martens2014insights, kunstner2019limitations}. In this section, we apply our theory to two standard second-order methods: damped Newton's method (DNM) and natural gradient descent (NGD). We provide more accurate results than those derived in \citet{liu2021noise}. The derivations are given in Appendix~\ref{app: second order}. For DNM, the preconditioning learning rate matrix is defined as $\Lambda:= \lambda A^{-1}$. The model fluctuation is shown to be proportional to the inverse of the Hessian: $\Sigma = \frac{\lambda \sigma^2}{gS -\lambda D}A^{-1}$, where $g:=2(1-\mu) - \left(\frac{1-\mu}{1+\mu}+\frac{1}{S}\right)\lambda$. The main difference with the previous results is that the fluctuation now depends explicitly on the dimension $D$, and implies a stability condition: $S\geq \lambda D /g$, corroborating the stability condition we derived above. For NGD, the preconditioning matrix is defined by the inverse of the Fisher information that $\Lambda:= \frac{\lambda}{S} J(\mathbf{w})^{-1} = \frac{\lambda}{S} C^{-1}$. We show that $\Sigma =\frac{\lambda}{2}\left(\frac{1}{1+D}\frac{1}{1+\mu}+\frac{1}{1-\mu}\frac{1}{S}\right)A^{-1}$ is one solution when $\sigma=0$, which also contains a correction related to $D$ compared to the result in \citet{liu2021noise} which is $\Sigma = \frac{\lambda}{2}\left(\frac{1}{1+\mu}+\frac{1}{1-\mu}\frac{1}{S}\right)A^{-1}$. A consequence is that $J\sim \Sigma^{-1}$. The surprising fact is that the stability of both NGD and DNM now crucially depends on $D$; combining with the results in Sec.~\ref{sec: high dimension linear regression}, this suggests that the dimension of the problem may crucially affect the stability and performance of the minibatch-based algorithms. This result also implies that some features we derived are shared across many algorithms that depend on minibatch noise and that our results may be relevant to a broad class of optimization algorithms other than SGD.



\vspace{-3mm}
\subsection{Failure of the $\lambda-S$ scaling law}\label{app sec: failure of the scaling law}
\vspace{-2mm}
One well-known technique in deep learning training is that one can scale $\lambda$ linearly as one increases the batch size $S$ to achieve high-efficiency training without hindering the generalization performance; however, it is known that this scaling law fails when the learning rate is too large, or the batch size is too small \citep{goyal2017accurate}. In \cite{Hoffer2017}, this scaling law is established on the ground that $\Sigma \sim \lambda/S$. However, our result in Theorem~\ref{thm: Sigma label} suggests the reason for the failure even for the simple setting of linear regression. Recall that the exact $\Sigma$ takes the form:
\[\Sigma  = \frac{\lambda \sigma^2}{S}\left( 1 + \frac{ \kappa_{\mu}}{S}\right)  G_{\mu}^{-1}\]
for a scalar $\lambda$. One notices that the leading term is indeed proportional to $\lambda/S$. However, the discrete-time SGD results in a second-order correction in $S$, and the term proportional to $1/S^2$ does not contain a corresponding $\lambda$; this explains the failure of the scaling law in small $S$, where the second-order contribution of $S$ becomes significant. To understand the failure at large $\lambda$, we need to look at the term $G_\mu$:
\[G_{\mu}=2(1-\mu)I_D-\left(\lambda\frac{1-\mu}{1+\mu}+\frac{\lambda}{S}\right) A.\]
One notices that the second term contains a part that only depends on $\lambda$ but not on $S$. This part is negligible compared to the first term when $\lambda$ is small; however, it becomes significant as the second term approaches the first term. Therefore, increasing $\lambda$ changes this part of the fluctuation, and the scaling law no more holds if $\lambda$ is large.

\begin{wrapfigure}{r!}{0.3\linewidth}
    \centering
    \vspace{-3em}
    \includegraphics[width=\linewidth]{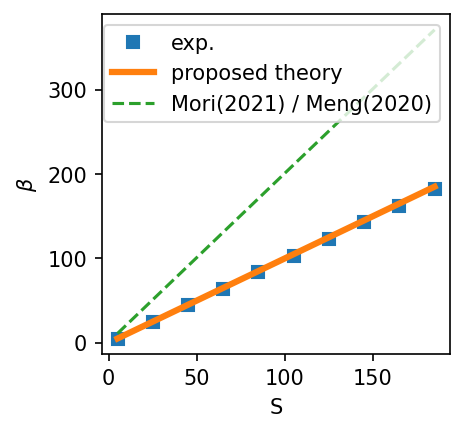}
    \vspace{-2em}
    \caption{\small Comparison of the proposed theory with the continuous-time theory on the SGD stationary distribution for $a\lambda=1$. The proposed theory agrees with the experiment exactly.}
    \label{fig:tail index main}
    \vspace{-2em}
\end{wrapfigure}
\vspace{-2mm}
\subsection{Power Law Tail in Discrete-time SGD}\label{sec: tail index}
\vspace{-1mm}
It has recently been discovered that the SGD noise causes a heavy-tail distribution \citep{Simsekli2019tailindex, simsekli2020hausdorff}, with a tail decaying like a power law with tail index $\beta$ \citep{hodgkinson2020multiplicative}.
In continuous-time, the stationary distribution has been found to obey a Student's t-like distribution, $p(w) \sim  L^{-(1 +\beta)/2}\sim\left(\sigma^2+a w^2\right)^{-(1 +\beta)/2}$  \citep{meng2020dynamic,mori2021logarithmic, wojtowytsch2021stochasticb}. However, this result is only established for continuous-time approximations to SGD and one does not know what affects the exponent $\beta$ for discrete-time SGD. Our result in Theorem~\ref{thm: Sigma label} can serve as a tool to find the discrete-time correction to the tail index of the stationary distribution. 
In Appendix~\ref{app: tail index proof}, we show that the tail index of discrete-time SGD in 1d can be estimated as $\beta(\lambda,S) = \frac{2S}{a\lambda}-S$. A clear discrete-time contribution is $-(S+1)$ which depends only on the batch size, while $\frac{2S}{a\lambda}+1$ is the tail index in the continuous-time limit \citep{mori2021logarithmic}. 
See Figure~\ref{fig:tail index main}; the proposed formula agrees with the experiment. Knowing the tail index $\beta$ is important for understanding the SGD dynamics because $\beta$ is equal to the smallest moment of $w$ that diverges. For example, when $\beta\leq 4$, then the kurtosis of $w$ diverges, and one expects to see outliers of $w$ very often during training; when $\beta\leq 2$, then the second moment of $w$ diverges, and one does not expect $w$ to converge in the minimum under consideration. Our result suggests that the discrete-time dynamics always leads to a heavier tail than the continuous-time theory expects, and therefore is more unstable.

\vspace{-3mm}
\section{Outlook}\label{sec: discussion}
\vspace{-2mm}

In this work, we have presented a systematic analysis with a focus on exactly solvable results to promote our fundamental understanding of SGD. One major limitation is that we have only focused on studying the asymptotic behavior of SGD in local minimum. For example, \cite{ziyin2022sgd} showed that SGD can converge to a local maximum when the learning rate is large. One important future step is thus to understand the SGD noise beyond a strongly convex landscape.

\section*{Acknowledgement}
Liu Ziyin thanks Jie Zhang, Junxia Wang, and Shoki Sugimoto. Ziyin is supported by the GSS Scholarship of The University of Tokyo. Kangqiao Liu was supported by the GSGC program of the University of Tokyo. This work was supported by KAKENHI Grant Numbers JP18H01145 and JP21H05185 from the Japan Society for the Promotion of Science.


\clearpage

\clearpage
\appendix


\clearpage
\section{Experiments}\label{app: extra exp}
\begin{figure}[tb!]
    \centering
    \includegraphics[width=0.45\linewidth]{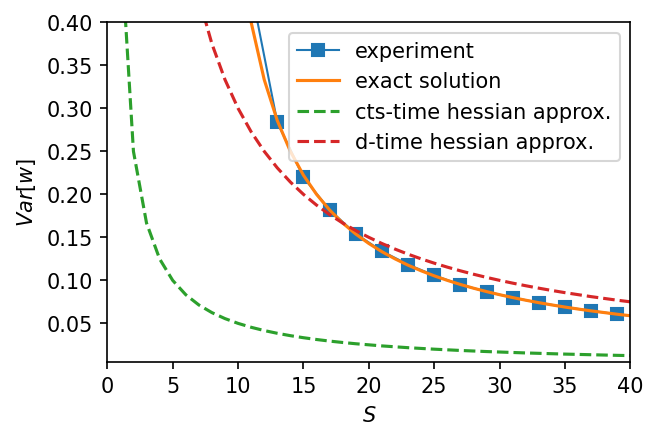}
    \includegraphics[width=0.54\linewidth]{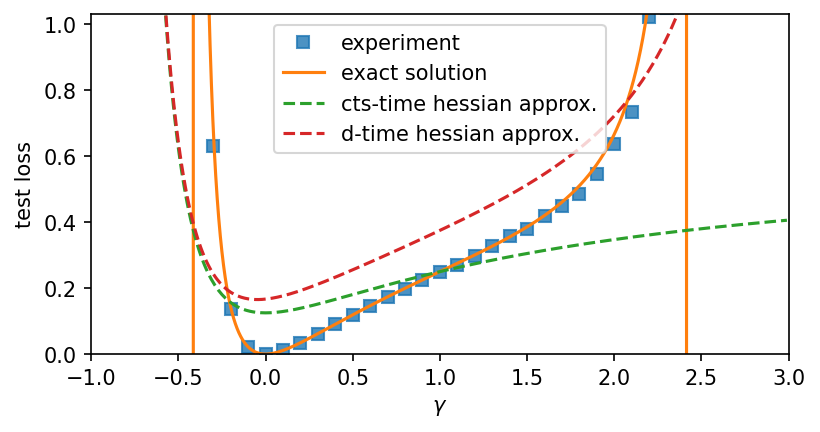}
    \vspace{-1em}
\caption{\textbf{Left}: 1d experiments with label noise. The parameters are set to be $a=1.5$ and $\lambda=1$. \textbf{Right}: Experiments with $L_2$ regularization with weight decay strength $\gamma$. The parameters are set to be $a=1$, $\lambda=0.5$, $S=1$. This is the standard case with a vanishing optimal $\gamma$. The vertical lines show where our theory predicts a divergence.}
\label{fig:label noise 1d}
\vspace{-1em}
\end{figure}

\subsection{Label noise and regularization}

Theorem~\ref{thm: Sigma label} can be verified empirically. We run 1d experiment in Figure~\ref{fig:label noise}(a) and high dimensional experiments in Figures~\ref{fig:label noise}(b)-(c), where we choose $D=2$ for visualization. We see that the continuous Hessian approximation fails badly for both large and small batch sizes. When the batch size is large, both the discrete-time Hessian approximation and our solution give a accurate estimate of the shape and the spread of the distribution. This suggests that when the batch size is large, discreteness is the determining factor of the fluctuation. When the batch size is small, the discrete Hessian approximation severely underestimates the strength of the noise. This reflects the fact that the isotropic noise enhancement is dominant at a small batch size.

\begin{figure}[tb!]
\captionsetup[subfigure]{justification=centering}
    \centering
    \begin{subfigure}{0.38\columnwidth}
    \includegraphics[width=\linewidth]{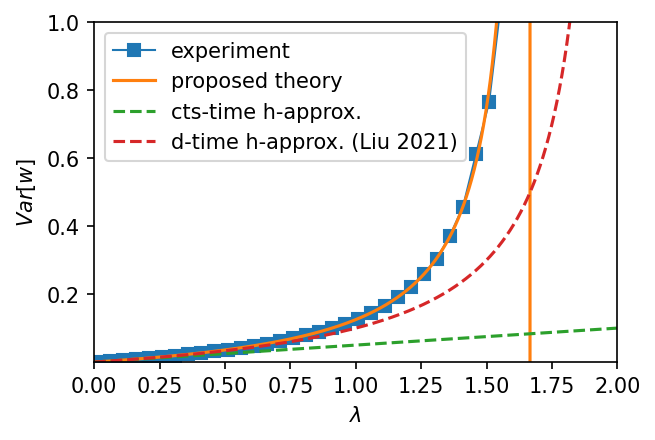}
    \vspace{-2em}
    \caption{$a=1$, $S=10$}
    \end{subfigure}
    \begin{subfigure}{0.30\columnwidth}
    \includegraphics[width=\linewidth]{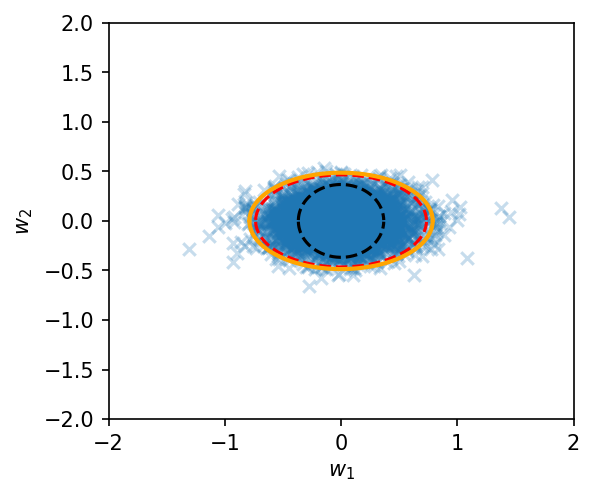}
    \vspace{-2em}
    \caption{$S=50$}
    \end{subfigure}
    \begin{subfigure}{0.29\columnwidth}
    \includegraphics[width=\linewidth]{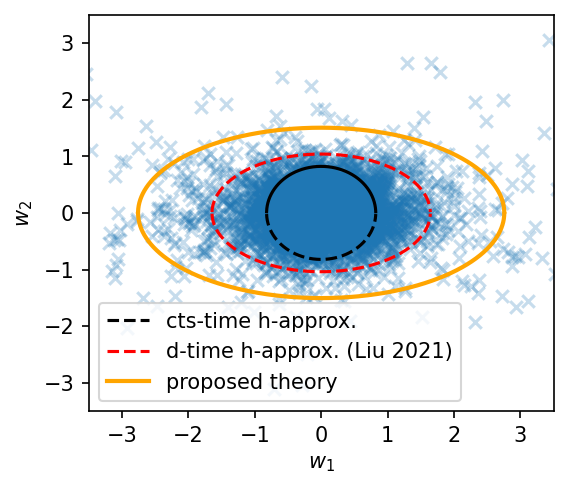}
    \vspace{-2em}
    \caption{ $S=10$}
    \end{subfigure}
    \vspace{-0.5em}
    \caption{Comparison between theoretical predictions and experiments. (a) 1d experiment. We plot $\Sigma$ as an increasing function of $\lambda$. We see that the continuous-time approximation fails to predict the divergence at a learning rate and the prediction in \citet{liu2021noise} severely underestimates the model fluctuation. In contrast, our result is accurate throughout the entire range of learning rates. (b)-(c) 2d experiments. The straight line shows where the proposed theory predicts a divergence in the variance, which agrees with experiment exactly. The Hessian has eigenvalues $1$ and $0.5$, and $\lambda=1.5$. For a large batch size, the discrete-time Hessian approximation is quite accurate; for a small $S$, the Hessian approximation underestimates the overall strength of the fluctuation. In contrast, the continuous-time result is both inaccurate in shape and in strength.}
    \label{fig:label noise}
    \vspace{-1em}
\end{figure}

In Figure~\ref{fig:label noise 1d}-Left, we run a 1d experiment with $\lambda=1$, $N=10000$ and $\sigma^2 = 0.25$. Comparing the predicted $\Sigma$, we see that the proposed theory agrees with the experiment across all ranges of $S$. The continuous theory with the Hessian approximation fails almost everywhere, while the recently proposed discrete theory with the Hessian approximation underestimates the fluctuation when $S$ is small. In Figure~\ref{fig:label noise 1d}-Right, we plot a standard case where the optimal regularization strength $\gamma$ is vanishing.

\begin{figure}[t!]
    \centering
    \includegraphics[width=0.45\linewidth]{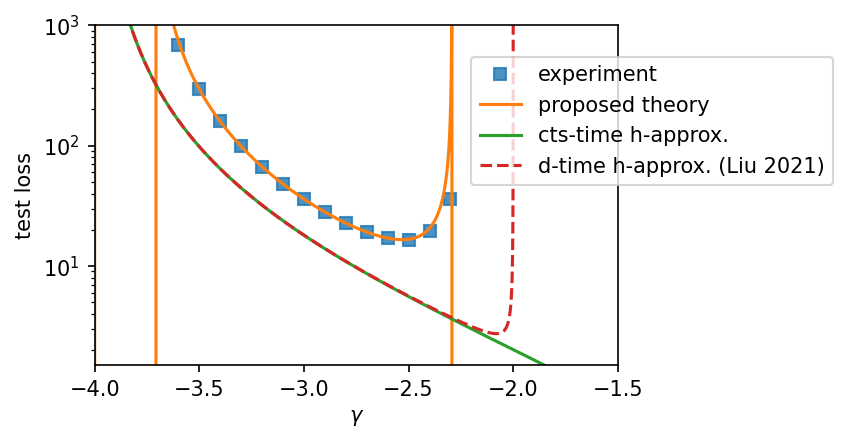}
\caption{1d experiments with $L_2$ regularization with weight decay strength $\gamma$. The parameters are set to be $a=4$, $\lambda=1$, $S=64$. This shows a case where the optimal $\gamma$ is negative. The vertical lines show where our theory predicts a divergence.}
\label{fig:regularization}
\end{figure}
Now, we validate the existence of the optimal negative weight decay as predicted by our formula. For illustration, we plot in Figure~\ref{fig:regularization} the test loss \eqref{eq: regular test loss} for a 1d example while varying either $S$ or $\lambda$. The orange vertical lines show the place where the theory predicts a divergence in the test loss. We also plot a standard case where the optimal $\gamma$ is close to 0 in Appendix~\ref{app: extra exp}. Also, we note that the proposed theory agrees better with the experiment.


\subsection{High-Dimensional Regression}\label{app sec: high dimension exp}
\begin{figure}[tb!]
    \centering
    \includegraphics[width=0.5\linewidth]{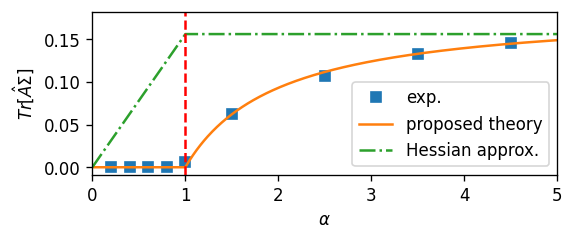}
    \vspace{-1em}
\caption{High-dimensional linear regression. We see that the predicted fluctuation coefficient agrees with the experiment well. The slight deviation is due to a finite training time and finite $N$ and $D$. On the other hand, a naive Hessian approximation results in a qualitatively wrong result.}
\label{fig:high dimension}
\vspace{-1em}
\end{figure}

See Figure~\ref{fig:high dimension}-Left. We vary $N$ with $D=1000$ held fixed. We set $\lambda=0.01$ and $S=32$. We see that the agreement between the theory and experiment is good, even for this modest dimension number $D$. The vertical line shows where the over-to-underparametrization transition takes place. As expected, there is no fluctuation when $\alpha<1$, and the fluctuation gradually increases as $\alpha \to \infty$. On the other hand, the Hessian approximation gives a wrong picture, predicting fluctuation to rise when there is no fluctuation and predicting a constant fluctuation just when the fluctuation starts to rise.

\clearpage

\begin{table*}[t!]
{\small
    \centering
    \caption{Comparison with previous results. For notational conciseness, we compare the case when all the relevant matrices commute. The model fluctuation $\Sigma$, the expected training loss $L_{\rm train}$ and the expected test loss $L_{\rm test}$ calculated by continuous- and discrete-time theories with Hessian approximation $C\approx H$ are presented. Exact solutions to these quantities obtained in the present work are shown in the rightmost column.}\vspace{-1mm}
    \label{tab:summary}
    
    \renewcommand\arraystretch{1.4}
    \resizebox{\textwidth}{!}{
    \begin{tabular}{c|c|c|c}
    \hline\hline
    & \multicolumn{2}{c|}{Hessian Approximation} & Exact Solution\\
    \hline
    & Cts-time Approximation & D-time Solution &  This Work \\
        \hline
    & { $\Sigma$} & { $\Sigma$} & { $\Sigma$}\\
    \hline
    
     Label Noise &  $\frac{\lambda}{2S}I_D$
     &  $\frac{\lambda}{S}(2I_D -\lambda A)^{-1}$
     &   $\frac{\lambda \sigma^2}{S}\left( 1 + \frac{\lambda  \kappa}{S}\right)  \left[2 I_D - \lambda \left(1 + \frac{1}{S}\right)A \right]^{-1}$ \\
     Input Noise &  $\frac{\lambda}{2S}I_D$
     & $\frac{\lambda}{S}(2I_D-\lambda K)^{-1}$
     & $\frac{\lambda {\rm Tr}[AK^{-1}BU]}{S}\left( 1 + \frac{\lambda  \kappa'}{S}\right)  \left[2 I_D - \lambda \left(1 + \frac{1}{S}\right)K \right]^{-1}$  \\
     
     $L_2$ Regularization  &  $\frac{\lambda}{2S}I_D$
     &   $\frac{\lambda}{S}(2I_D-\lambda K)^{-1}$
     &  Eq.~\eqref{eq: regularization solution} \\
    \hline\hline
    
    & { $L_{\rm train}$} & {$L_{\rm train}$} & {$L_{\rm train}$} \\
    \hline
    
     Label Noise &  $\frac{\lambda}{4S}{\rm Tr}[A]+\frac{1}{2}\sigma^2$
     &  Eq.~\eqref{eq: label training discrete K approx}
     &  $\frac{\sigma^2}{2}\left(1+\frac{\lambda \kappa}{S}\right)$ \\
     Input Noise &  $\frac{\lambda}{4S}{\rm Tr}[K]+\frac{1}{2}{\rm Tr}[AK^{-1}BU]$
     &  Eq.~\eqref{eq: input training discrete K approx}
     &  $\frac{1}{2}{\rm Tr}[AK^{-1}BU]\left(1+\frac{\lambda }{S}\kappa'\right)$ \\
     $L_2$ Regularization  &  $\frac{\lambda}{4S}{\rm Tr}[K]+\frac{1}{2}{\rm Tr}[AK^{-1}\Gamma U]$
     &   Eq.~\eqref{eq: Kapprox training}
     &  Eq.~\eqref{eq: regularization training loss} \\
     \hline\hline
    
    & { $L_{\rm test}$} & {$L_{\rm test}$} & {$L_{\rm test}$} \\
    \hline
    
     Label Noise &  $\frac{\lambda}{4S}{\rm Tr}[A]$
     &  $\frac{\lambda}{2S}{\rm Tr}[A(2I_D-\lambda A)^{-1}]$
     &  $\frac{\lambda\sigma^2}{2S}\kappa$ \\
     Input Noise &  $\frac{\lambda}{4S}{\rm Tr}[A]+\frac{1}{2}{\rm Tr}[B'^{\rm T}AB'U]$
     &  Eq.~\eqref{eq: input test discrete K approx}
     &  $\frac{\lambda}{2S}{\rm Tr}[AK^{-1}BU] \kappa'+\frac{1}{2}{\rm Tr}[B'^{\rm T}AB'U]$ \\
     $L_2$ Regularization  &  $\frac{\lambda}{4S}{\rm Tr}[A]+\frac{1}{2}{\rm Tr}[AK^{-2}\Gamma^2 U]$
     &   Eq.~\eqref{eq: Kapprox test}
     &  Eq.~\eqref{eq: regular test loss} \\
     \hline\hline
     
    \end{tabular}}
  } \vspace{-1em}
\end{table*}

\section{Comparison with Conventional Hessian Approximation}\label{app: comparison}
We compare our results for the three cases with the results obtained with the conventional Hessian approximation of the noise covariance, i.e., $C\approx H$, where $H$ is the Hessian of the loss function.  We summarize the analytical results for a special case in Table~\ref{tab:summary}.

\subsection{Label Noise}
We first consider discrete-time dynamics with the Hessian approximation. The matrix equation is
\begin{align}
    \Sigma A + A\Sigma -\lambda A \Sigma A=\frac{\lambda}{S} A.
\end{align}
Compared with the exact result \eqref{eq: preconditioning matrix eq}, it is a large-$S$ limit up to the constant $\sigma^2$. This constant factor is ignored during the approximation that $J(\mathbf{w}):=\mathbb{E}_{\rm B}[\nabla l\nabla l^{\rm T}]\approx \mathbb{E}_{\rm B}[\nabla \nabla^{\rm T} l]:=H(\mathbf{w})$, which is exact only when $l(\{x_i\},\mathbf{w})$ is a \textit{negative log likelihood} function of $\mathbf{w}$. Solving the matrix equation yields
\begin{align}
    \Sigma = \frac{\lambda}{S}(2I_D -\lambda A)^{-1}.
\end{align}
The training loss and the test loss are 
\begin{align}
    &L_{\rm train}=\frac{\lambda}{2S}{\rm Tr}[A(2I_D-\lambda A)^{-1}]+\frac{1}{2}\sigma^2,\label{eq: label training discrete K approx}\\
    &L_{\rm test}=\frac{\lambda}{2S}{\rm Tr}[A(2I_D-\lambda A)^{-1}].\label{eq: label test discrete K approx}
\end{align}
On the other hand, by taking the large-$S$ limit directly from the exact equation \eqref{eq: preconditioning matrix eq}, the factor $\sigma^2$ is present:
\begin{align}
    \Sigma A + A\Sigma -\lambda A \Sigma A=\frac{\lambda}{S}\sigma^2 A.
\end{align}

For the continuous-time limit with the Hessian approximation, the matrix equation is
\begin{align}
    \Sigma A + A\Sigma=\frac{\lambda}{S} A,
\end{align}
which is the small-$\lambda$ limit up to the factor $\sigma^2$. The variance is
\begin{align}
    \Sigma=\frac{\lambda}{2S}I_D.
\end{align}
The training and the test error are
\begin{align}
    &L_{\rm train}=\frac{\lambda}{4S}{\rm Tr}[A]+\frac{1}{2}\sigma^2,\\
    &L_{\rm test}=\frac{\lambda}{4S}{\rm Tr}[A].
\end{align}
Again, taking the small-$\lambda$ limit directly from the exact result \eqref{eq: preconditioning matrix eq} shows the presence of the factor $\sigma^2$ on the right hand side of the matrix equation.

\subsection{Input Noise}
The case with the input noise is similar to the label noise. This can be understood if we replace $A$ by $K$ and $\sigma^2$ by ${\rm Tr}[AK^{-1}BU]$. The model parameter variance resulting from the discrete-time dynamics under the Hessian approximation is
\begin{align}
    \Sigma=\frac{\lambda}{S}(2I_D-\lambda K)^{-1}.
\end{align}
The training and the test error are 
\begin{align}
      &L_{\rm train}=\frac{\lambda}{2S}{\rm Tr}[K(2I_D-\lambda K)^{-1}]+\frac{1}{2}{\rm Tr}[AK^{-1}BU],\label{eq: input training discrete K approx} \\
    &L_{\rm test}=\frac{\lambda}{2S}{\rm Tr}[A(2I_D-\lambda K)^{-1}]+\frac{1}{2}{\rm Tr}[B'^{\rm T}AB'U].\label{eq: input test discrete K approx}
\end{align}
The large-$S$ limit from the exact matrix equation \eqref{eq: matrix eq of input} results in a prefactor ${\rm Tr}[AK^{-1}BU]$ in the fluctuation:
\begin{align}
    \Sigma=\frac{\lambda}{S}{\rm Tr}[AK^{-1}BU](2I_D-\lambda K)^{-1}.
\end{align}

For the continuous-time limit, we take $\lambda\to 0$. The Hessian approximation gives
\begin{align}
    &\Sigma=\frac{\lambda}{2S}I_D,\\
    &L_{\rm train}=\frac{\lambda}{4S}{\rm Tr}[K]+\frac{1}{2}{\rm Tr}[AK^{-1}BU],\\
    &L_{\rm test}=\frac{\lambda}{4S}{\rm Tr}[A]+\frac{1}{2}{\rm Tr}[B'^{\rm T}AB'U].
\end{align}
The large-$S$ limit again produces a prefactor ${\rm Tr}[AK^{-1}BU]$.

\subsection{$L_2$ Regularization}\label{sec: L2 regularization}
For learning with regularization, there is a more difference between the Hessian approximation and the limit taken directly from the exact theory. We first adopt the Hessian approximation for the discrete-time dynamics. The matrix equation is
\begin{align}
    \Sigma K + K\Sigma-\lambda K \Sigma K = \frac{\lambda}{S}K,
\end{align}
which is similar to the previous subsection. However, it is different from the large-$S$ limit of the exact matrix equation \eqref{eq: matrixeq of regular}:
\begin{align}
    \Sigma K + K\Sigma-\lambda K \Sigma K =\frac{\lambda}{S}\left({\rm Tr}[AK^{-2}\Gamma^2U]A+AK^{-1}\Gamma U\Gamma K^{-1}A\right).
\end{align}
This significant difference suggests that the conventional Fisher-to-Hessian approximation $J\approx H$ fails badly. The fluctuation, the training loss, and the test loss with the Hessian approximation are
\begin{align}
    &\Sigma=\frac{\lambda}{S}(2I_D-\lambda K)^{-1}, \label{eq: Kapprox Sigma}\\
    &L_{\rm train}=\frac{\lambda}{2S}{\rm Tr}[K(2I_D -\lambda K)^{-1}]+\frac{1}{2}{\rm Tr}[AK^{-1}\Gamma U], \label{eq: Kapprox training}\\
    &L_{\rm test}=\frac{\lambda}{2S}{\rm Tr}[A(2I_D -\lambda K)^{-1}]+\frac{1}{2}{\rm Tr}[AK^{-2}\Gamma^2 U],\label{eq: Kapprox test}
\end{align}
while the large-$S$ limit of the exact theory yields
\begin{align}
    &\Sigma=\frac{\lambda}{S}{\rm Tr}[AK^{-2}\Gamma^2U]AK^{-1}(2I_D-\lambda K)^{-1}+\frac{\lambda}{S}A^2 K^{-3} \Gamma^2 (2I_D-\lambda K)^{-1} U, \label{eq: limit Sigma}\\
    &L_{\rm train}=\frac{\lambda}{2S}{\rm Tr}[AK^{-2}\Gamma^2U]{\rm Tr}[A(2I_D-\lambda K)^{-1}] + \frac{\lambda}{2S}{\rm Tr}[A^2 K^{-2}\Gamma^2(2I_D-\lambda K)^{-1}U] \nonumber\\&\quad\quad\quad+ \frac{1}{2}{\rm Tr}[AK^{-1}\Gamma U],\label{eq: limit training}\\
    &L_{\rm test}=\frac{\lambda}{2S}{\rm Tr}[AK^{-2}\Gamma^2U]{\rm Tr}[A(2I_D-\lambda K)^{-1}] + \frac{\lambda}{2S}{\rm Tr}[A^3 K^{-3}\Gamma^2(2I_D-\lambda K)^{-1}U] \nonumber\\&\quad\quad\quad+ \frac{1}{2}{\rm Tr}[AK^{-2}\Gamma^2 U].\label{eq: limit test}
\end{align}

The continuous-time results are obtained by taking the small-$\lambda$ limit on Eqs.~\eqref{eq: Kapprox Sigma}-\eqref{eq: Kapprox test} for the Hessian approximation and on Eqs.~\eqref{eq: limit Sigma}-\eqref{eq: limit test} for the limiting cases of the exact theory. Specifically, for the Hessian approximation, we have
\begin{align}
    &\Sigma=\frac{\lambda}{2S}I_D, \\
    &L_{\rm train}=\frac{\lambda}{4S}{\rm Tr}[K]+\frac{1}{2}{\rm Tr}[AK^{-1}\Gamma U],\\
    &L_{\rm test}=\frac{\lambda}{4S}{\rm Tr}[A]+\frac{1}{2}{\rm Tr}[AK^{-2}\Gamma^2 U].
\end{align}
The small-$\lambda$ limit of the exact theory yields
\begin{align}
    &\Sigma=\frac{\lambda}{2S}{\rm Tr}[AK^{-2}\Gamma^2U]AK^{-1}+\frac{\lambda}{2S}A^2 K^{-3} \Gamma^2  U, \\
    &L_{\rm train}=\frac{\lambda}{4S}{\rm Tr}[AK^{-2}\Gamma^2U]{\rm Tr}[A] + \frac{\lambda}{4S}{\rm Tr}[A^2 K^{-2}\Gamma^2U] + \frac{1}{2}{\rm Tr}[AK^{-1}\Gamma U],\\
    &L_{\rm test}=\frac{\lambda}{4S}{\rm Tr}[AK^{-2}\Gamma^2U]{\rm Tr}[A] + \frac{\lambda}{4S}{\rm Tr}[A^3 K^{-3}\Gamma^2U] + \frac{1}{2}{\rm Tr}[AK^{-2}\Gamma^2 U].
\end{align}

\clearpage
\section{Proof of the General Formula}\label{app: general}
\subsection{Proof of Theorem~\ref{theo: general formula} and Corollary~\ref{corr: different minima}}
We restate the theorem.
\begin{theorem}
    Let the training loss be $L_\Gamma = L_0 + \frac{1}{2}\mathbf{w}^{\rm T}\Gamma \mathbf{w}$ and the models be optimized with SGD in the neighborhood of a local minimum $\mathbf{w}^*$. When $\Gamma \neq 0$, the noise covariance is given by 
    \begin{equation}
        C = \frac{2L_0(\mathbf{w}^*)}{S}H_0(\mathbf{w}^*)-\frac{1}{S}\Gamma\mathbf{w}^*\mathbf{w}^{*\mathrm{T}}\Gamma + O(S^{-2}) + O(|\mathbf{w} - \mathbf{w}^*|^2). 
    \end{equation}
    When $\Gamma = 0$ and $L_0(\mathbf{w}^*) \neq 0$,
    \begin{equation}
        C = \frac{2L_0(\mathbf{w}^*)}{S}H_0(\mathbf{w}^*) + O(S^{-2}) + O(|\mathbf{w} - \mathbf{w}^*|^2).
    \end{equation}
    When $\Gamma = 0$ and $L_0(\mathbf{w}^*) = 0$,
    \begin{equation}
        C = \frac{1}{S} \left({\rm Tr}[H_0(\mathbf{w}^*) \Sigma] I_D - H_0(\mathbf{w}^*) \Sigma\right) H_0(\mathbf{w}^*) + O(S^{-2}) + O(|\mathbf{w} - \mathbf{w}^*|^2).
    \end{equation}
\end{theorem}

\begin{proof}
 We will use the following shorthand notations: $\ell_i:=\ell(f(\mathbf{w},x_i),y_i)$, $\ell_i':=\frac{\partial\ell_i}{\partial f}$, $\ell_i'':=\frac{\partial^2\ell_i}{\partial f^2}$. 
The Hessian of the loss function without regularization $H_0(\mathbf{w})=\nabla\nabla^\mathrm{T}L_0(\mathbf{w})$ is given by
\begin{equation}
    H_0(\mathbf{w})=\frac{1}{N}\sum_{i=1}^N\ell_i''\nabla f(\mathbf{w},x_i)\nabla f(\mathbf{w},x_i)^\mathrm{T}
    +\frac{1}{N}\sum_{i=1}^N\ell_i'\nabla\nabla^\mathrm{T}f(\mathbf{w},x_i).
    \label{eq:general_Hessian}
\end{equation}
The last term of Eq.~(\ref{eq:general_Hessian}) can be ignored when $L_0 \ll 1$, since 
\begin{align}
    \left\|\frac{1}{N}\sum_{i=1}^N\ell_i'\nabla\nabla^\mathrm{T} f(\mathbf{w},x_i)\right\|_F
    &\leq\left(\frac{1}{N}\sum_{i=1}^N(\ell_i')^2\right)^{1/2}\left(\frac{1}{N}\sum_{i=1}^N\|\nabla\nabla^\mathrm{T}f(\mathbf{w},x_i)\|_F^2\right)^{1/2} \nonumber \\
    &=
    \langle\ell'^2\rangle^{1/2}\left(\frac{1}{N}\sum_{i=1}^N\|\nabla\nabla^\mathrm{T}f(\mathbf{w},x_i)\|_F^2\right)^{\frac{1}{2}} \nonumber, \\
    &= 
    \sqrt{2 L_0(\mathbf{w})} \left(\frac{1}{N}\sum_{i=1}^N\|\nabla\nabla^\mathrm{T}f(\mathbf{w},x_i)\|_F^2\right)^{\frac{1}{2}} \nonumber,
\end{align}
where $\|\cdot\|_F$ stands for the Frobenius norm\footnote{In the linear regression problem, the last term of Eq.~(\ref{eq:general_Hessian}) does not exist since $\nabla\nabla^\mathrm{T}f(\mathbf{w},x_i)=0$.}, and we have defined the variable $\langle\ell'^2\rangle:=\frac{1}{N}\sum_{i=1}^N(\ell_i')^2$.
Since $\ell_i''=1$ for the mean-square error, we obtain
\begin{equation}
H_0(\mathbf{w}) = \frac{1}{N}\sum_{i=1}^N\nabla f(\mathbf{w},x_i)\nabla f(\mathbf{w},x_i)^\mathrm{T} + O \left(\sqrt{L_0}\right)
\end{equation}
near a minimum.
The Hessian with regularization $H_\Gamma(\mathbf{w})=\nabla\nabla^\mathrm{T}L_\Gamma(\mathbf{w})$ is just given by $H_0(\mathbf{w})+\Gamma$.


On the other hand, the SGD noise covariance $C(\mathbf{w})$ is given by Eq.~(\ref{eq:noise_cov}).
By assumption 2, the SGD noise covariance is directly related to the Hessian: 
\begin{align}
C(\mathbf{w})& = 
\frac{\langle\ell'^2\rangle}{SN}\sum_{i=1}^N\nabla f(\mathbf{w},x_i)\nabla f(\mathbf{w},x_i)^\mathrm{T}
-\frac{1}{S}\nabla L_{\Gamma}(\mathbf{w})\nabla L_{\Gamma}(\mathbf{w})^\mathrm{T} \nonumber \\& \quad+ \frac{2}{SN} \sum_{i=1}^N (\ell_i - L_0 )\nabla f(\mathbf{w},x_i)\nabla f(\mathbf{w},x_i)^\mathrm{T}
\nonumber \\
& = \frac{2L_0(\mathbf{w})}{S}H_0(\mathbf{w})-\frac{1}{S}\nabla L_{\Gamma}(\mathbf{w})\nabla L_{\Gamma}(\mathbf{w})^\mathrm{T} + o(L_0).
\end{align}%
This finishes the proof. $\square$

Now we prove Corollary~\ref{corr: different minima}.

\textit{Proof}. Near a minimum $\mathbf{w}^*$ of the full loss $L_\Gamma(\mathbf{w})$, we have
\begin{equation}
    \nabla L_{\Gamma}(\mathbf{w}) = H_0(\mathbf{w}^*)(\mathbf{w}-\mathbf{w}^*)+\Gamma\mathbf{w}^* + O(|\mathbf{w}-\mathbf{w}^*|^2),
    \label{eq:nabla_L0}
\end{equation}
within the approximation $L_\Gamma(\mathbf{w}) = L_\Gamma(\mathbf{w}^*)+(1/2)(\mathbf{w}-\mathbf{w}^*)^\mathrm{T}H_\Gamma(\mathbf{w}^*)(\mathbf{w}-\mathbf{w}^*)^\mathrm{T} + O(|\mathbf{w}-\mathbf{w}^*|^2)$. Equations~(\ref{eq:cov_Hessian}) and (\ref{eq:nabla_L0}) give the SGD noise covariance near a minimum of $L_\Gamma(\mathbf{w})$.

Now it is worth discussing two different cases separately: (1) with regularization and (2) without regularization. We first discuss the case when regularization is present.
In this case, the regularization $\Gamma$ is not small enough,  and the SGD noise covariance is \emph{not} proportional to the Hessian.
Near a local or global minimum $\mathbf{w}\approx \mathbf{w}^*$, the first term of the right-hand side of Eq.~(\ref{eq:nabla_L0}) is negligible, and hence we obtain\vspace{-1mm}
\begin{align}
     \mathbb{E}_{\mathbf{w}}[C(\mathbf{w})] &= \frac{2L_0(\mathbf{w}^*)}{S}H_0(\mathbf{w}^*)-\frac{1}{S}\Gamma\mathbf{w}^*\mathbf{w}^{*\mathrm{T}}\Gamma \nonumber \\
     &\quad +  \mathbb{E}_{\mathbf{w}} \left[\frac{1}{S}H_0(\mathbf{w}^*)(\mathbf{w}-\mathbf{w}^*)(\mathbf{w}-\mathbf{w}^*)^\mathrm{T}H_0(\mathbf{w}^*) \right] + O(|\mathbf{w} - \mathbf{w}^*|^2) \nonumber \\
     &= \frac{2L_0(\mathbf{w}^*)}{S}H_0(\mathbf{w}^*)-\frac{1}{S}\Gamma\mathbf{w}^*\mathbf{w}^{*\mathrm{T}}\Gamma  + O(S^{-2}) + O(|\mathbf{w} - \mathbf{w}^*|^2).
 \end{align}
where we have used the fact that $\mathbb{E}[\mathbf{w}]=\mathbf{w}^*$. The SGD noise does not vanish even at a global minimum of $L_\Gamma(\mathbf{w})$. Note that this also agrees with the exact result derived in Sec.~\ref{sec: regularization}: together with an anisotropic noise that is proportional to the Hessian, a rank-$1$ noise proportional to the strength of the regularization appears. This rank-1 noise is a signature of regularization.

On the other hand, as we will see below, the SGD noise covariance is proportional to the Hessian near a minimum when there is no regularization, i.e., $\Gamma = 0$.
We have
\begin{align}
    C(\mathbf{w}) &= \frac{2L_0(\mathbf{w})}{S}H_0(\mathbf{w})
    -\frac{1}{S}H_0(\mathbf{w}^*)(\mathbf{w}-\mathbf{w}^*)(\mathbf{w}-\mathbf{w}^*)^\mathrm{T}H_0(\mathbf{w}^*) + O(|\mathbf{w} - \mathbf{w}^*|^2). \label{eq:C_no_reg}
\end{align}%
For this case, we need to differentiate between a local minimum and a global minimum.
When $L_0(\mathbf{w}^*)$ is not small enough (e.g. at a local but not global minimum),
\begin{align}
    C(\mathbf{w})  
    & = \frac{2L_0(\mathbf{w}^*)}{S}H_0(\mathbf{w}) + \frac{(\mathbf{w}-\mathbf{w}^*)^\mathrm{T}H_0(\mathbf{w}^*)(\mathbf{w}-\mathbf{w}^*)}{S}H_0(\mathbf{w}) \nonumber \\
     &\quad -\frac{1}{S}H_0(\mathbf{w}^*)(\mathbf{w}-\mathbf{w}^*)(\mathbf{w}-\mathbf{w}^*)^\mathrm{T}H_0(\mathbf{w}^*) + O(|\mathbf{w} - \mathbf{w}^*|^2) \nonumber \\ 
    & = \frac{2L_0(\mathbf{w}^*)}{S}H_0(\mathbf{w}) + O(S^{-2}) + O(|\mathbf{w} - \mathbf{w}^*|^2) \nonumber \\
    & = \frac{2L_0(\mathbf{w}^*)}{S}H_0(\mathbf{w}^*) + O(S^{-2}) + O(|\mathbf{w} - \mathbf{w}^*|^2),
\end{align}
and so, to leading order,
\begin{equation}
    C = \frac{2L_0(\mathbf{w}^*)}{S}H_0(\mathbf{w}^*),
\end{equation}
which is proportional to the Hessian but also proportional to the achievable approximation error.

On the other hand, when $L_0(\mathbf{w}^*)$ is vanishingly small (e.g. at a global minimum), we have $2L_0(\mathbf{w})\approx (\mathbf{w}-\mathbf{w}^*)^\mathrm{T}H_0(\mathbf{w}^*)(\mathbf{w}-\mathbf{w}^*)$, and thus obtain
\begin{align}
    C(\mathbf{w}) &= \frac{1}{S}\Bigl[(\mathbf{w}-\mathbf{w}^*)^\mathrm{T}H_0(\mathbf{w}^*)(\mathbf{w}-\mathbf{w}^*)H_0(\mathbf{w}^*)
    -H_0(\mathbf{w}^*)(\mathbf{w}-\mathbf{w}^*)(\mathbf{w}-\mathbf{w}^*)^\mathrm{T}H_0(\mathbf{w}^*)\Bigr]\nonumber\\
    &\quad + O(S^{-2}) + O(|\mathbf{w} - \mathbf{w}^*|^2),
    \label{eq:cov_global}
\end{align}
i.e.,
\begin{equation}
    \mathbb{E}[C] =  \frac{1}{S} \left({\rm Tr}[H_0 \Sigma]I_D - H_0 \Sigma\right) H_0+ O(S^{-2}) + O(|\mathbf{w} - \mathbf{w}^*|^2).
\end{equation}
This completes the proof.
\end{proof}

\begin{remark}
    It should be noted that the second term on the right-hand side of Eq.~(\ref{eq:cov_global}) would typically be much smaller than the first term for large $D$. 
For example, when $H_0(\mathbf{w}^*)=aI_D$ with $a>0$,
the first and the second terms are respectively given by $(a^2/S)\|\mathbf{w}-\mathbf{w}^*\|^2I_D$ and $-(a^2/S)(\mathbf{w}-\mathbf{w}^*)(\mathbf{w}-\mathbf{w}^*)^\mathrm{T}$.
The Frobenius norm of the former is given by $(Da^2/S)\|\mathbf{w}-\mathbf{w}^*\|^2$, while that of the latter is given by $(a^2/S)\|\mathbf{w}-\mathbf{w}^*\|^2$, which indicates that in Eq.~(\ref{eq:cov_global}), the first term is dominant over the second term for large $D$.
Therefore the second term of Eq.~(\ref{eq:cov_global}) can be dropped for large $D$, and Eq.~(\ref{eq:cov_global}) is simplified as
\begin{equation}
    \begin{cases}
        C(\mathbf{w})\approx\frac{(\mathbf{w}-\mathbf{w}^*)^\mathrm{T}H_0(\mathbf{w}^*)(\mathbf{w}-\mathbf{w}^*)}{S}H_0(\mathbf{w}^*) ; \\
        
        \mathbb{E}[C] \approx \frac{{\rm Tr}[H_0 \Sigma]}{S} H_0 .
    \end{cases}
    \label{eq:cov_global_simplified}
\end{equation}
Again, the SGD noise covariance is proportional to the Hessian.
\end{remark}

In conclusion, as long as the regularization is small enough, that the SGD noise covariance near a minimum is proportional to the Hessian is a good approximation. This implies that the noise is multiplicative, which is known to lead to a heavy tail distribution \citep{clauset2009power, levy1996power}. Thus, we have studied the nature of the minibatch SGD noise in three different situations. 
As an example, we have demonstrated the power of this general formulation by applying it to the high-dimensional linear regression problem in Sec.~\ref{sec: high dimension linear regression}.

\subsection{Proof of Corollary~\ref{cor: general Sigma}}
\begin{proof}
 We prove the case where $\Gamma =0$ and $L(\mathbf{w}^*)\ne 0$ as an example. Substituting Theorem~\ref{theo: general formula} into Theorem~\ref{theo: liu2021}yields
 \begin{equation}
     \left[2I_D - \frac{1}{1+\mu}\Lambda H_0\right]\Lambda H_0 \Sigma = \frac{2L_0}{S(1-\mu)}\Lambda^2 H_0,
 \end{equation}
 where we have assumed necessary commutation relations. Suppose that the Hessian $H_0$ is of rank-$r$ with $r \le D$. The singular-value decomposition and its Moore-Penrose pseudo inverse are given by $H_0 = USV^{\rm T}$ and $H_0^+ = VS^{+}U^{\rm T}$, respectively, where $U$ and $V$ are unitary, $S$ is a rank-$r$ diagonal matrix with elements being singular values of $H_0$, and $S^+$ is obtained by inverting every non-zero entry of $S$. Multiplying $H_0^+$ to both sides of the above equation, we have
 \begin{equation}
     P_r \Sigma = \frac{2L_0}{S(1-\mu)}P_r\Lambda\left(2I_D - \frac{\Lambda}{1+\mu}H_0\right)^{-1},
 \end{equation}
 where $P_r = {\rm diag}(1,1,\dots,1,0,\dots,0)$ is the projection operator with $r$ non-zero entries. When the Hessian is full-rank, i.e., $r=D$, the Moore-Penrose pseudo inverse is nothing but the usual inverse. The other cases can be calculated similarly.
\end{proof}

\clearpage
\section{Applications}\label{app: implications proof}
\subsection{Infinite-dimensional Limit of the Linear Regression Problem}
Now we apply the general theory in Sec.~\ref{sec:general} to linear regressions in the high-dimensional limit, namely  $N,D\to\infty$ with $\alpha:=N/D$ held fixed.

\subsubsection{Proof of Proposition~\ref{prop: high D C}}\label{app: high D C}
The loss function
\begin{equation}
    L(\mathbf{w})=\frac{1}{2N}\sum_{i=1}^N\left(\mathbf{w}^\mathrm{T}x_i-y_i\right)^2
\end{equation}
with $y_i=\mathbf{u}^\mathrm{T}+\epsilon_i$ can be written as
\begin{equation}
    L(\mathbf{w})=\frac{1}{2}\left(\mathbf{w}-\mathbf{u}-\hat{A}^+\mathbf{v}\right)^\mathrm{T}\hat{A}\left(\mathbf{w}-\mathbf{u}-\hat{A}^+\mathbf{v}\right)-\frac{1}{2}\mathbf{v}^\mathrm{T}\hat{A}^+\mathbf{v}+\frac{1}{2N}\sum_{i=1}^N\epsilon_i^2,
\end{equation}
where $\hat{A}:=\frac{1}{N}\sum_{i=1}^Nx_ix_i^\mathrm{T}$ is an empirical covariance for the training data and $\mathbf{v}:=\frac{1}{N}\sum_{i=1}^Nx_i\epsilon_i$. The symbol $(\cdot)^+$ denotes the Moore-Penrose pseudoinverse. We also introduce the the averaged traing loss: $L_{\rm train}:= \mathbb{E}_{\mathbf{w} }[L(\mathbf{w})]$

The minimum of the loss function is given by
\begin{equation}
    \mathbf{w}^*=\mathbf{u}+\hat{A}^+\mathbf{v}+\Pi\mathbf{r},
\end{equation}
where $\mathbf{r}\in\mathbb{R}^D$ is an arbitrary vector and $\Pi$ is the projection onto the null space of $\hat{A}$.
Since $1-\Pi=\hat{A}^+\hat{A}$, $\mathbf{w}^*$ is also expressed as
\begin{equation}
    \mathbf{w}^*=\hat{A}^+(\hat{A}\mathbf{u}+\mathbf{v})+\Pi\mathbf{r}.
\end{equation}
In an underparameterized regime $\alpha>1$, $\Pi=0$ almost surely holds as long as the minimum eigenvalue of $A$ (not $\hat{A}$) is positive~\citep{hastie2019surprises}.
In this case, $\hat{A}^+=\hat{A}^{-1}$ and we obtain
\begin{equation}
    \mathbf{w}^*=\mathbf{u}+\hat{A}^{-1}\mathbf{v} \quad\text{for }\alpha>1.
\end{equation}
On the other hand, in an overparameterized regime $\alpha>1$, $\Pi\neq 0$ and there are infinitely many global minima.
In the ridgeless regression, we consider the global minimum that has the minimum norm $\|\mathbf{w}^*\|$, which corresponds to
\begin{equation}
    \mathbf{w}^*=\hat{A}^+(\hat{A}\mathbf{u}+\mathbf{v})=(1-\Pi)\mathbf{u}+\hat{A}^+\mathbf{v} \quad\text{for ridgeless regression with }\alpha<1.
\end{equation}
In both cases, the loss function is expressed as
\begin{equation}
    L(\mathbf{w})=\frac{1}{2}\left(\mathbf{w}-\mathbf{w}^*\right)^\mathrm{T}\hat{A}\left(\mathbf{w}-\mathbf{w}^*\right)-\frac{1}{2}\mathbf{v}^\mathrm{T}\hat{A}\mathbf{v}+\frac{1}{2N}\sum_{i=1}^N\epsilon_i^2.
\end{equation}

Asymptotically, $\mathbf{w}_t$ converges to a stationary point $\mathbf{w}^*$ with fluctuation $\Sigma$ obeying the following equation (Theorem~\ref{theo: liu2021}:
\begin{equation}
    \lambda\hat{A}\Sigma+\lambda\Sigma\hat{A}-\lambda^2\hat{A}\Sigma\hat{A}=\lambda^2C.
    \label{eq:Sigma_high_dim}
\end{equation}

The SGD noise covariance $C$ is given by Eq.~(\ref{eq:cov_Hessian}).
In the present case, the Hessian is given by $H=\hat{A}$ and we also have
\begin{equation}
    \frac{1}{N}\sum_{i=1}^N(\ell_i')^2
    =\frac{1}{N}\sum_{i=1}^N\left(\mathbf{w}^\mathrm{T}x_i-y_i\right)^2
    =\frac{2}{N}\sum_{i=1}^N\ell_i
    =2L(\mathbf{w}).
\end{equation}
On the other hand, $\nabla L(\mathbf{w})\nabla L(\mathbf{w})^\mathrm{T}=\hat{A}(\mathbf{w}-\mathbf{w}^*)(\mathbf{w}-\mathbf{w}^*)^\mathrm{T}\hat{A}$, and hence $\mathbb{E}_{\mathbf{w}}[\nabla L(\mathbf{w})\nabla L(\mathbf{w})^\mathrm{T}]=\hat{A}\Sigma\hat{A}$.
Therefore we obtain
\begin{equation}
C=\mathbb{E}_{\mathbf{w}}[C(\mathbf{w})]=\frac{2L_\mathrm{train}}{S}\hat{A}-\frac{1}{S}\hat{A}\Sigma\hat{A}.
\label{eq:C_high_dim}
\end{equation}

Now, we find $L_\mathrm{train}$. First, we define $X\in\mathbb{R}^{N\times D}$ as $X_{ik}=(x_i)_k$, and $\vec{\epsilon}\in\mathbb{R}^N$ as $\vec{\epsilon}_i=\epsilon_i$.
Then $\mathbf{w}^*=(1-\Pi)\mathbf{u}+\hat{A}^+\mathbf{v}=(1-\Pi)\mathbf{u}+(X^\mathrm{T}X)^+X^\mathrm{T}\vec{\epsilon}$.

With this notation, we have $\hat{A}=X^\mathrm{T}X/N$, and the loss function is expressed as
\begin{equation}
    L(w)=\frac{1}{2}(\mathbf{w}-\mathbf{w}^*)^\mathrm{T}\hat{A}(\mathbf{w}-\mathbf{w}^*)-\frac{1}{2N}\vec{\epsilon}^\mathrm{T}X(X^\mathrm{T}X)^+X^\mathrm{T}\vec{\epsilon}+\frac{1}{2N}\sum_{i=1}^N\epsilon_i^2.
\end{equation}
We therefore obtain
\begin{equation}
    L_\mathrm{train}=\frac{1}{2}\mathrm{Tr}[\hat{A}\Sigma]-\frac{1}{2N}\mathbb{E}[\vec{\epsilon}^\mathrm{T}X(X^\mathrm{T}X)^+X^\mathrm{T}\vec{\epsilon}]+\frac{\sigma^2}{2}.
\end{equation}
Here,
\begin{equation}
    \mathbb{E}[\vec{\epsilon}^\mathrm{T}X(X^\mathrm{T}X)^+X^\mathrm{T}\vec{\epsilon}]
    =\sigma^2\mathrm{Tr}[(X^\mathrm{T}X)(X^\mathrm{T}X)^+]
    =\sigma^2\mathrm{Tr}(1-\Pi).
    \label{eq:L_training_calc}
\end{equation}
We can prove that the following identity is almost surely satisfied~\citep{hastie2019surprises} as long as the smallest eigenvalue of $A$ (not $\hat{A}$) is positive:
\begin{equation}
    \mathrm{Tr}(1-\Pi)=\min\{D,N\}.
\end{equation}
We therefore obtain
\begin{equation}
    L_\mathrm{train}=\frac{1}{2}\mathrm{Tr}[\hat{A}\Sigma]-\frac{\sigma^2}{2N}\min\{D,N\}+\frac{\sigma^2}{2}
    =\left\{\begin{aligned}
    &\frac{1}{2}\mathrm{Tr}[\hat{A}\Sigma]+\frac{1}{2}\left(1-\frac{1}{\alpha}\right)\sigma^2 & \text{for }\alpha>1, \\
    &\frac{1}{2}\mathrm{Tr}[\hat{A}\Sigma] & \text{for }\alpha\leq 1
    \end{aligned}
    \right.
    \label{eq:L_training_high_dim}
\end{equation}

By substituting Eq.~(\ref{eq:L_training_high_dim}) into Eq.~(\ref{eq:C_high_dim}), we obtain the following SGD noise covariance:
\begin{equation}
    C=\left\{\begin{aligned}
    &\frac{1}{S}\left(\mathrm{Tr}[\hat{A}\Sigma]-\hat{A}\Sigma\right)\hat{A}+\frac{\sigma^2}{S}\left(1-\frac{1}{\alpha}\right)\hat{A} &\text{for }\alpha>1, \\
    &\frac{1}{S}\left(\mathrm{Tr}[\hat{A}\Sigma]-\hat{A}\Sigma\right)\hat{A} &\text{for }\alpha\leq 1.
    \end{aligned}
    \right.
    \label{eq:C_high_dim_result}
\end{equation}
This finishes the proof. $\square$


\subsubsection{Proof of Theorem~\ref{thm: trace of high D}}\label{app: trace high D}
\begin{proof}
We have to solve this equation:
\begin{align}
    \hat{A}\Sigma + \Sigma\hat{A}-\lambda\hat{A} \Sigma \hat{A}=\lambda C,
\end{align}
where $C$ is given in Proposition~\ref{prop: high D C}. Using the similar trick of multiplying by $\hat{G}:=2I_D-\lambda\left(1-\frac{1}{S}\right)\hat{A}$ as in Appendix~\ref{app: der of Sigma label}, one obtains
\begin{equation}
        {\rm Tr}[\hat{A}\Sigma] = \begin{cases}
        \frac{\lambda\sigma^2}{S}\left(1-\frac{1}{\alpha}\right)\hat{\kappa}&\text{for }\alpha>1;\\
            0 &\text{for }\alpha\leq 1,
        \end{cases}
    \end{equation}
    where   $\hat{\kappa}:=\frac{{\rm Tr}[\hat{G}^{-1}\hat{A}]}{1-\frac{\lambda}{S}{\rm Tr}[\hat{G}^{-1}\hat{A}]}$ with $\hat{G}:=2I_D-\lambda\left(1-\frac{1}{S}\right)\hat{A}$.
    
Substituting the above trace into the matrix equation, we have
\begin{equation}
        \Sigma = \begin{cases}
        \frac{\lambda\sigma^2}{S}\left(1-\frac{1}{\alpha}\right)\left(1+\frac{\lambda}{S}\hat{\kappa}\right)\hat{G}^{-1}&\text{for }\alpha>1;\\
            0 &\text{for }\alpha\leq 1.
        \end{cases}
    \end{equation}
\end{proof}

\subsection{Second-order Methods}\label{app: second order}
\begin{proposition}\label{prop: DNM}
Suppose that we run DNM with $\Lambda:= \lambda A^{-1}$ with random noise in the label. The model fluctuation is
\begin{equation}
    \Sigma = \frac{\lambda \sigma^2}{gS -\lambda D}A^{-1},
\end{equation}
where $g := 2(1-\mu) - \left(\frac{1-\mu}{1+\mu}+\frac{1}{S}\right)\lambda$.
\end{proposition}
\begin{proof}
Substituting $\Lambda= \lambda A^{-1}$ into Eqs.~\eqref{eq: preconditioning matrix eq} and \eqref{eq: label noise covariance} yields
\begin{equation}
    g\Sigma = \frac{\lambda}{S}\left({\rm Tr} [A\Sigma] + \sigma^2\right)A^{-1},
\end{equation}
where $g:== 2(1-\mu) - \left(\frac{1-\mu}{1+\mu}+\frac{1}{S}\right)\lambda$. Multiplying $A$ and taking trace on both sides, we have
\begin{equation}
    {\rm Tr} [A\Sigma] = \frac{\lambda D \sigma^2}{gS - \lambda D}.
\end{equation}
Therefore, the model fluctuation is
\begin{equation}
    \Sigma = \frac{\lambda \sigma^2}{gS -\lambda D}A^{-1}.
\end{equation}
\end{proof}

\begin{proposition}\label{prop: NGD}
Suppose that we run NGD with $\Lambda:= \frac{\lambda}{S} J(\mathbf{w})^{-1}\approx \frac{\lambda}{S} C^{-1}$ with random noise in the label. The model fluctuation is 
\begin{align}
\Sigma = \left[\frac{\lambda}{4}g-\frac{1}{2}\frac{\sigma^2}{1+D}+\frac{1}{4}\sqrt{\lambda^2 g^2 + 4\lambda \left(g-\frac{2}{1+D}\frac{1}{1+\mu}\right)\frac{\sigma^2}{1+D}+4\left(\frac{\sigma^2}{1+D}\right)^2}\right]A^{-1},
\end{align}
where $g:= \frac{1}{1+D}\frac{1}{1+\mu}+\frac{1}{1-\mu}\frac{1}{S}$.
\end{proposition}
\begin{proof}
 Similarly to the previous case, the matrix equation satisfied by $\Sigma$ is \vspace{-1.2em}
 
{\small
\begin{equation}
    (1-\mu)(C^{-1}A\Sigma + \Sigma A C^{-1})-\frac{1+\mu^2}{1-\mu^2}\frac{\lambda}{S} C^{-1} A \Sigma A C^{-1}+\frac{\mu}{1-\mu^2}\frac{\lambda}{S}(C^{-1} A C^{-1} A \Sigma+ \Sigma A C^{-1} A C^{-1})=\frac{\lambda}{S}C^{-1}.\vspace{-1.2em}
\end{equation}}

Although it is not obvious how to directly solve this equation, it is possible to guess one solution according to the hope that $\Sigma$ be proportional to $J^{-1}$, in turn, $A^{-1}$ \citep{Amari:1998:NGW:287476.287477,liu2021noise}. We assume that $\Sigma = x A^{-1}$ and substitute it into the above equation to solve for $x$. This yields one solution without claiming its uniqueness. By simple algebra, this $x$ is solved to be
\begin{equation}
    x = \frac{\lambda}{4}g-\frac{1}{2}\frac{\sigma^2}{1+D}+\frac{1}{4}\sqrt{\lambda^2 g^2 + 4\lambda \left(g-\frac{2}{1+D}\frac{1}{1+\mu}\right)\frac{\sigma^2}{1+D}+4\left(\frac{\sigma^2}{1+D}\right)^2}.
\end{equation}
Let $\sigma=0$. We obtain the result in Sec.~\ref{sec: second order}.
\end{proof}

\subsection{Estimation of Tail Index}\label{app: tail index proof}

In \citet{mori2021logarithmic, meng2020dynamic}, it is shown that the (1d) discrete-time SGD results in a distribution that is similar to a Student's t-distribution:
\begin{equation}\label{eq: heavy tail distirbution of sgd}
    p(w) \sim (\sigma^2 + aw^2)^{-\frac{1+\beta}{2}},
\end{equation}
where $\sigma^2$ is the degree of noise in the label, and $a$ is the local curvature of the minimum. For large $w$, this distribution is a power-law distribution with tail index:
\begin{equation}
    p(|w|) \sim |w|^{-(1 + \beta)},
\end{equation}
and it is not hard to check that $\beta$ also equal to the smallest moment of $w$ that diverges: $\mathbb{E}[w^{\beta}] = \infty$. Therefore, estimating $\beta$ can be of great use both empirically and theoretically.

In continuous-time, it is found that $\beta_{\rm cts}=\frac{2S}{a\lambda}+1$ \citep{mori2021logarithmic}. For discrete-time SGD, we hypothesize that the discrete-time nature causes a change in the tail index $\beta= \beta_{\rm cts} + \epsilon$, and we are interested in finding $\epsilon$. We propose a ``semi-continuous" approximation to give the formula to estimate the tail index. Notice that Theorem~\ref{thm: Sigma label} gives the variance of the discrete-time SGD, while Eq.~\eqref{eq: heavy tail distirbution of sgd} can be integrated to give another value of the variance, and the two expressions must be equal for consistency. This gives us an equation that $\beta$ must satisfy:
\begin{equation}
     \int p(w; \beta) (w - \mathbb{E}[w])^2 = \text{Var}[w],
\end{equation}
this procedure gives the following formula:
\begin{equation}
    \beta(\lambda,S)=\frac{2S}{a\lambda}-S = \beta_{\rm cts}+ \epsilon,
\end{equation}
and one immediately recognizes that $-(S+1)$ is the discrete-time contribution to the tail index. See Figure~\ref{fig:my_label} for additional experiments. We see that the proposed formula agrees with the experimentally measured value of the tail index for all ranges of the learning rate, while the result of \cite{mori2021logarithmic} is only correct when $\lambda\to 0^+$.
\cite{hodgkinson2020multiplicative} also studies the tail exponent of discrete-time SGD; however, their conclusion is only that the ``index decreases with the learning rate and increases with the batch size". In contrast, our result give the functional form of the tail index directly. In fact, this is the first work that gives any functional form for the tail index of discrete-time SGD fluctuation to the best of our knowledge.

The following proposition gives the intermediate steps in the calculation.

\begin{figure}
    \centering
    \includegraphics[width=0.32\linewidth]{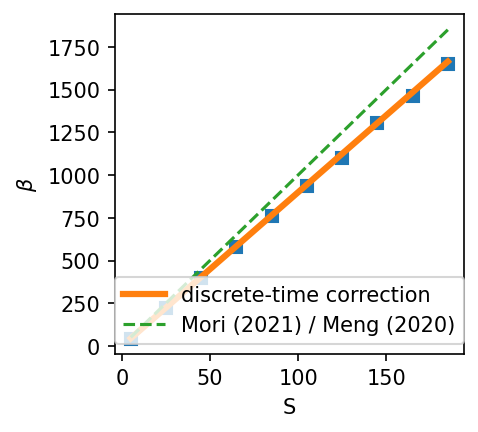}
    \includegraphics[width=0.32\linewidth]{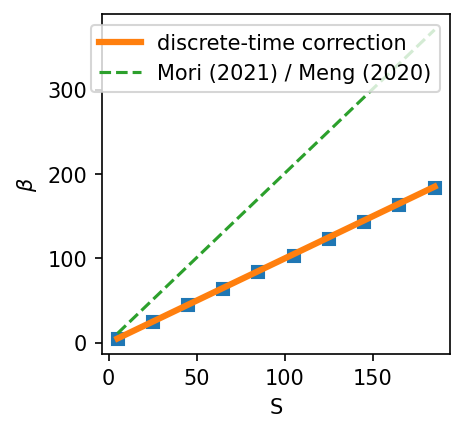}
    \includegraphics[width=0.32\linewidth]{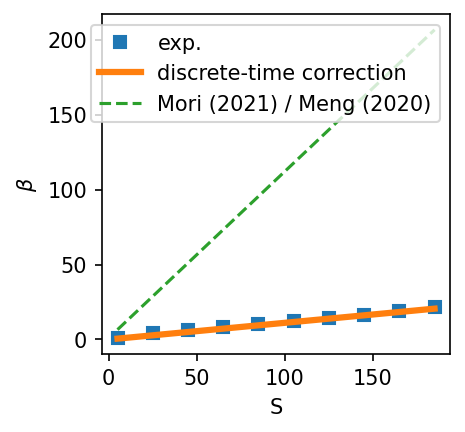}
    \caption{Tail index $\beta$ of the stationary distribution of SGD in a 1d linear regression problem. \textbf{Left to Right}: $a\lambda=0.2,\ 1.0,\ 1.8$.}
    \label{fig:my_label}
\end{figure}

\begin{proposition}\label{thm: tail idx stu}$($Tail index estimation for discrete-time SGD$)$ Let the parameter distribution be $p(w) \sim \left(\sigma^2+a  w^2\right)^{-\frac{1+\beta}{2}}$, and $\text{Var}[w]$ be given by Theorem~\ref{thm: Sigma label}. Then
\begin{equation}\label{eq: tail idx stu}
    \beta(\lambda,S)=\frac{2S}{a\lambda}-S.
\end{equation}
\end{proposition}
\begin{proof}


The normalization factor for the distribution exists if $\beta>0$:
\begin{equation}
    \mathcal{N} = \sqrt{\frac{a}{\pi}}\frac{\sigma^{\beta}\Gamma(\frac{1+\beta}{2})}{\Gamma(\frac{\beta}{2})}.
\end{equation}
If $\beta>2$, the variance exists and the value is
\begin{equation}\label{eq: var stu}
    {\rm Var}[w] = \frac{\sigma^2}{a(\beta-2)}.
\end{equation}
By equating Eq.~\eqref{eq: var stu} with the exact variance \eqref{eq: solution}, we are able to solve for an expression of the tail index as
\begin{equation}
    \beta(\lambda,S)  = \frac{2S}{a\lambda}-S.
\end{equation}
\end{proof}

\clearpage

\section{Proofs and Additional Theoretical Considerations}\label{app: all proofs}
\subsection{Proof of Proposition~\ref{prop: noise_cov}}\label{app: der of C}
The with-replacement sampling is defined in Definition~\ref{def: with replacement}. Let us here define the without-replacement sampling.
\begin{definition}\label{def: without replacement}
    A minibatch SGD \textit{without replacement} computes the update to the parameter $\mathbf{w}$ with the following set of equations:
    \vspace{-1mm}
    \begin{equation}\label{eq: without replacement}
        \begin{cases}
            \hat{\mathbf{g}}_t= \frac{1}{S}\sum_{i\in B_t} \nabla  \ell(x_i, y_i, \mathbf{w}_{t-1});\\
            \mathbf{w}_t = \mathbf{w}_{t-1} - \lambda \hat{\mathbf{g}}_t,
        \end{cases}
    \end{equation}
    where $S := |B_t|\leq N$ is the minibatch size, and
    the set $B_t$ is an element uniformly-randomly drawn from the set of all $S$-size subsets of $\{1,...,N\}$.
   %
\end{definition} 
From the definition of the update rule for sampling with or without replacement, the covariance matrix of the SGD noise can be exactly derived.\vspace{-1mm}
\begin{proposition}\label{prop: finite N covariance}The covariance matrices of noise in SGD due to minibatch sampling as defined in Definitions~\ref{def: with replacement} and \ref{def: without replacement} with an arbitrary $N$ are
\begin{align}
    C(\mathbf{w})= \begin{cases}
        \frac{1}{S}\left[\frac{1}{N}\sum_{i=1}^{N}\nabla \ell_i \nabla \ell_i^{\rm T}-\nabla L(\mathbf{w})\nabla L(\mathbf{w})^{\rm T} \right], &(\text{with\ replacement})\\
        \frac{N-S}{S(N-1)}\left[\frac{1}{N}\sum_{i=1}^{N}\nabla \ell_i \nabla \ell_i^{\rm T}-\nabla L(\mathbf{w})\nabla L(\mathbf{w})^{\rm T}\right], &(\text{without\ replacement})
    \end{cases}\vspace{-1mm}
\end{align}
where the shorthand notation $\ell_i(\mathbf{w}):= l(x_i, y_i, \mathbf{w})$ is used.
\end{proposition}

 In the limit of $S=1$ or $N\gg S$, two cases coincide. In the $N\gg S$ limit, both methods of sampling have the same noise covariance as stated in Proposition~\ref{prop: noise_cov}:
\begin{align}
    C(\mathbf{w})=\frac{1}{SN}\sum_{i=1}^{N}\nabla \ell_i \nabla \ell_i^{\rm T}-\frac{1}{S}\nabla L(\mathbf{w})\nabla L(\mathbf{w})^{\rm T}.
\end{align}

\begin{remark}
We also note that a different way of defining minibatch noise exists in \citet{Hoffer2017}. The difference is that our definition requires the size of each minibatch to be exactly $S$, while \citet{Hoffer2017} treats the batch size also as a random variable and is only expected to be $S$. In comparison, our definition agrees better with the common practice.
\end{remark}

Now we prove Proposition~\ref{prop: finite N covariance}. 
\begin{proof}
We derive the noise covariance matrices for sampling with and without replacement. We first derive the case with replacement. According to the definition, the stochastic gradient for sampling with replacement can be rewritten as
\begin{align}
    \hat{\mathbf{g}}=\frac{1}{S}\sum_{n=1}^{N}\mathbf{g}_ns_n,
\end{align}
where $\mathbf{g}_n:=\nabla \ell_n$ and
\begin{align}
    s_n=l,\ {\rm if\ } l-{\rm multiple\ } n'{\rm s\ are\ sampled\ in\ S,\ with\ }0\le l\le S.
\end{align}
The probability of $s_n$ assuming value $l$ is given by the multinomial distribution
\begin{align}
    P(s_n=l)={{S}\choose l}\left(\frac{1}{N}\right)^l \left(1-\frac{1}{N}\right)^{S-l}.
\end{align}
Therefore, the expectation value of $s_n$ is given by
\begin{align}
    \mathbb{E}_{\rm B}[s_n]=\sum_{l=0}^{S}lP(s_n=l)=\frac{S}{N},
\end{align}
which gives
\begin{align}
    \mathbb{E}_{\rm B}[\hat{\mathbf{g}}]=\mathbf{g}:=\frac{1}{N}\sum_{n=1}^{N}\mathbf{g}_n=\nabla L(\mathbf{w}).
\end{align}
For the covariance, we first calculate the covariance between $s_n$ and $s_{n'}$. Due to the properties of the covariance of multinomial distribution, we have for $n\ne n'$
\begin{align}
    \mathbb{E}_{\rm B}[s_ns_{n'}]&={\rm cov}[s_n ,s_{n'}]+\mathbb{E}[s_n]^2\nonumber\\
    &=-\frac{S}{N^2}+\frac{S^2}{N^2}\nonumber\\
    &=\frac{S(S-1)}{N^2};
\end{align}
and for $n=n'$
\begin{align}
    \mathbb{E}_{\rm B}[s_ns_{n}]&={\rm Var}[s_n]+\mathbb{E}[s_n]^2\nonumber\\
    &= \frac{S}{N}\frac{N-1}{N}+\frac{S^2}{N^2}\nonumber\\
    &=\frac{SN+S(S-1)}{N^2}.
\end{align}
Substituting these results into the definition of the noise covariance yields
\begin{align}
    C(\mathbf{w}) &=  \mathbb{E}_{\rm B}[\hat{\mathbf{g}}\hat{\mathbf{g}}^{\rm T}]-\mathbb{E}_{\rm B}[\hat{\mathbf{g}}]\mathbb{E}_{\rm B}[\hat{\mathbf{g}}]^{\rm T}\nonumber\\
    &= \frac{1}{S^2}\sum_{n=1}^{N}\sum_{n'=1}^{N}\mathbf{g}_n\mathbf{g}_{n'}^{\rm T}\mathbb{E}_{\rm B}[s_ns_{n'}]-\mathbf{g}\mathbf{g}^{\rm T}\nonumber\\
    &=  \frac{1}{S^2}\sum_{n, n'=1}^{N}\mathbf{g}_n\mathbf{g}_{n'}^{\rm T}\frac{S(S-1)}{N^2} +  \frac{1}{S^2}\sum_{n=1}^{N}\mathbf{g}_n\mathbf{g}_{n}^{\rm T}\left[\frac{SN+S(S-1)}{N^2}-\frac{S(S-1)}{N^2}\right]-\mathbf{g}\mathbf{g}^{\rm T}\nonumber\\
    &= \frac{1}{NS}\sum_{n=1}^{N}\mathbf{g}_n\mathbf{g}_{n}^{\rm T}-\frac{1}{S}\mathbf{g}\mathbf{g}^{\rm T}
    \nonumber\\
    &= \frac{1}{S}\left[\frac{1}{N}\sum_{i=1}^{N}\nabla \ell_i \nabla \ell_i^{\rm T}-\nabla L(\mathbf{w})\nabla L(\mathbf{w})^{\rm T} \right].
\end{align}

Then, we derive the noise covariance for sampling without replacement. Similarly, according to the definition, the stochastic gradient for sampling without replacement can be rewritten as
\begin{align}
    \hat{\mathbf{g}}=\frac{1}{S}\sum_{n=1}^{N}\mathbf{g}_ns_n,
\end{align}
where 
\begin{align}
    s_n=\begin{cases}
        0, {\rm if\ } n\notin S;\\
        1, {\rm if\ } n\in S.
    \end{cases}
\end{align}
The probability of $n$ that is sampled in $S$ from $N$ is given by
\begin{align}
    P(s_n=1)=\frac{{{N-1}\choose {S-1}}}{{{N}\choose{S}}}=\frac{S}{N}.
\end{align}
The expectation value of $s_n$ is then given by
\begin{align}
    \mathbb{E}_{\rm B}[s_n]=P(s_n=1)=\frac{S}{N},
\end{align}
which gives
\begin{align}
    \mathbb{E}_{\rm B}[\hat{\mathbf{g}}]=\mathbf{g}:=\frac{1}{N}\sum_{n=1}^{N}\mathbf{g}_n=\nabla L(\mathbf{w}).
\end{align}
For the covariance, we first calculate the covariance between $s_n$ and $s_{n'}$. By definition, we have for $n\ne n'$
\begin{align}
    \mathbb{E}_{\rm B}[s_ns_{n'}]&=P(s_n=1,s_n'=1)=P(s_n=1|s_n'=1)P(s_n'=1)\nonumber\\
    &=\frac{{{N-2}\choose{S-2}}}{{{N-1}\choose{S-1}}}\frac{{{N-1}\choose{S-1}}}{{{N}\choose{S}}}=\frac{S(S-1)}{N(N-1)};
\end{align}
and for $n=n'$
\begin{align}
    \mathbb{E}_{\rm B}[s_ns_{n}]=P(s_n=l)=\frac{S}{N}.
\end{align}
Substituting these results into the definition of the noise covariance yields
\begin{align}
    C(\mathbf{w}) &=  \mathbb{E}_{\rm B}[\hat{\mathbf{g}}\hat{\mathbf{g}}^{\rm T}]-\mathbb{E}_{\rm B}[\hat{\mathbf{g}}]\mathbb{E}_{\rm B}[\hat{\mathbf{g}}]^{\rm T}\nonumber\\
    &= \frac{1}{S^2}\sum_{n=1}^{N}\sum_{n'=1}^{N}\mathbf{g}_n\mathbf{g}_{n'}^{\rm T}\mathbb{E}_{\rm B}[s_ns_{n'}]-\mathbf{g}\mathbf{g}^{\rm T}\nonumber\\
    &=  \frac{1}{S^2}\sum_{n, n'=1}^{N}\mathbf{g}_n\mathbf{g}_{n'}^{\rm T}\frac{S(S-1)}{N(N-1)} +  \frac{1}{S^2}\sum_{n=1}^{N}\mathbf{g}_n\mathbf{g}_{n}^{\rm T}\left[\frac{S}{N}-\frac{S(S-1)}{N(N-1)}\right]-\mathbf{g}\mathbf{g}^{\rm T}\nonumber\\
    &= \frac{1}{NS}\frac{N-S}{N-1}\sum_{n=1}^{N}\mathbf{g}_n\mathbf{g}_{n}^{\rm T}-\frac{N-S}{S(N-1)}\mathbf{g}\mathbf{g}^{\rm T}
    \nonumber\\
    &= \frac{N-S}{S(N-1)}\left[\frac{1}{N}\sum_{i=1}^{N}\nabla \ell_i \nabla \ell_i^{\rm T}-\nabla L(\mathbf{w})\nabla L(\mathbf{w})^{\rm T} \right].
\end{align}
\end{proof}

\subsection{Proofs in Sec.~\ref{sec: label noise}}\label{app: proofs in sec label}
\subsubsection{Proof of Lemma~\ref{prop: label noise covariance}}\label{app: der of label noise cov}
\begin{proof}
From the definition of noise covariance \eqref{eq:noise_cov}, the covariance matrix for the noise in the label is
{\small
\begin{align}
    C(\mathbf{w})&=\frac{1}{NS} \sum_{i=1}^{N}\nabla l_i(\mathbf{w}_{t-1}) \nabla l_i(\mathbf{w}_{t-1})^{\rm T} - \frac{1}{S}\nabla L(\mathbf{w}_{t-1})\nabla L(\mathbf{w}_{t-1})^{\rm T}\nonumber\\
    &= \frac{1}{S} \frac{1}{N}\sum_i^N(\mathbf{w}^{\rm T}x_i - \epsilon_i)x_i x_i^{\rm T}(\mathbf{w}^{\rm T}x_i  - \epsilon_i)^{\rm T} - \frac{1}{S}\left[\frac{1}{N}\sum_i^N (\mathbf{w}^{\rm T}x_i - \epsilon_i)x_i\right]\left[\frac{1}{N}\sum_j^N x_j^{\rm T}(\mathbf{w}^{\rm T}x_j  - \epsilon_j)^{\rm T}\right]\nonumber\\
    &= \frac{1}{S} \frac{1}{N}\sum_i^N(\mathbf{w}^{\rm T} x_ix_i x_i^{\rm T}x_i^{\rm T} \mathbf{w} + \epsilon_i^2 x_ix_i^{\rm T}  ) - \frac{1}{S}\left[\frac{1}{N}\sum_i^N(\mathbf{w}^{\rm T} x_ix_i)\right]\left[\frac{1}{N}\sum_j^N(x_i^{\rm T}x_i^{\rm T}\mathbf{w})\right] \label{eq: C sum}\\
    &= \frac{1}{S}(A\mathbf{w}\mathbf{w}^{\rm T}A + {\rm Tr}[A\mathbf{w}\mathbf{w}^{\rm T}]A + \sigma^2 A  ), \label{eq: C large number}
\end{align}}
where we have invoked the law of large numbers and the expectation value of the product of four Gaussian random variables in the third line is evaluated as follows.

Because $N$ is large, we invoke the law of large numbers to obtain the $(j,k)$-th component of the matrix as
\begin{align}
    \lim_{N\to \infty}\frac{1}{N}\sum_{i=1}^{N}(\mathbf{w}^{\rm T}x_i x_i x_i^{\rm T}x_i^{\rm T}\mathbf{w})_{jk}&=\mathbb{E}_{\rm B}[\mathbf{w}^{\rm T}x x x^{\rm T}x^{\rm T}\mathbf{w}]_{jk}=\mathbb{E}_{\rm B}\left[\sum_{i}^{D}w_ix_i x_j x_k \sum_{i'}^{D}x_{i'}w_{i'}\right].
\end{align}
Because the average is taken with respect to $x$ and each $x$ is a Gaussian random variable, we apply the expression for the product of four Gaussian random variables $\mathbb{E}[x_1 x_2 x_3 x_4]=\mathbb{E}[x_1 x_2]\mathbb{E}[x_3 x_4]+\mathbb{E}[x_1 x_3]\mathbb{E}[x_2 x_4]+\mathbb{E}[x_1 x_4]\mathbb{E}[x_2 x_3]-2\mathbb{E}[x_1] \mathbb{E}[x2]\mathbb{E}[x_3]\mathbb{E}[ x_4]$ \citep{Janssen1988} to obtain
\begin{align}
    &\mathbb{E}_{\rm B}\left[\sum_{i}^{D}w_ix_i x_j x_k \sum_{i'}^{D}x_{i'}w_{i'}\right]\nonumber\\
    &=\mathbb{E}_{\rm B}\left[\sum_{i}^{D}w_ix_i x_j\right]\mathbb{E}_{\rm B}\left[x_k \sum_{i'}^{D}x_{i'}w_{i'}\right]+\mathbb{E}_{\rm B}\left[\sum_{i}^{D}w_ix_ix_k\right]\mathbb{E}_{\rm B}\left[x_j\sum_{i'}^{D}x_{i'}w_{i'}\right]\nonumber\\
    &\quad +\mathbb{E}_{\rm B}\left[\sum_{i}^{D}w_ix_i\sum_{i'}^{D}x_{i'}w_{i'}\right]\mathbb{E}_{\rm B}\left[x_j x_k\right]\nonumber\\
    &=2(A\mathbf{w}\mathbf{w}^{\rm T}A)_{jk}+{\rm Tr}[A\mathbf{w}\mathbf{w}^{\rm T}]A_{jk}.
\end{align}
Writing $\Sigma:=\mathbb{E}_{\mathbf{w}}[\mathbf{w}\mathbf{w}^{\rm T}]$, we obtain
\begin{equation}
   \mathbb{E}_\mathbf{w}[C(\mathbf{w})] =: C   = \frac{1}{S}(A\Sigma A + {\rm Tr}[A\Sigma]A+ \sigma^2 A).
\end{equation}
\end{proof}
This method has been utilized repeatedly in this work.

\subsubsection{Proof of Theorem~\ref{thm: Sigma label}}\label{app: der of Sigma label}
\begin{proof}
We substitute Eq.~\eqref{eq: label noise covariance} into Eq.~\eqref{eq: preconditioning matrix eq} which is a general solution obtained in a recent work \citep{liu2021noise}:
\begin{equation}
    (1-\mu) (\Lambda A\Sigma + \Sigma A \Lambda) - \frac{1+\mu^2}{1-\mu^2}\Lambda A\Sigma A \Lambda+ \frac{\mu}{1-\mu^2}(\Lambda A\Lambda A\Sigma +\Sigma A\Lambda A\Lambda)  = \Lambda C\Lambda.
\end{equation}
To solve it, we assume the commutation relation that $[\Lambda, A]:= \Lambda A - A \Lambda =0 $. Therefore, the above equation can be alternatively rewritten as
\begin{align}
    &\left [(1-\mu) I_D - \frac{1}{2}\left(\frac{1-\mu}{1+\mu}+\frac{1}{S}\right)\Lambda A\right] \Sigma A \Lambda + \Lambda A\Sigma \left [(1-\mu) I_D - \frac{1}{2}\left(\frac{1-\mu}{1+\mu}+\frac{1}{S}\right)\Lambda A\right] \nonumber\\
    &- \frac{1}{S}{\rm Tr}[A\Sigma]\Lambda A = \frac{1}{S} \sigma^2 \Lambda A.\label{eq: matrixeq label 2}
\end{align}

To solve this equation, we first need to solve for ${\rm Tr}[A\Sigma]$. Multiplying Eq.~\eqref{eq: matrixeq label 2} by $G_{\mu}^{-1}: = \left[2(1-\mu)I_D-\left(\frac{1-\mu}{1+\mu}+\frac{1}{S}\right)\Lambda A\right]^{-1}$ and taking trace, we obtain
\begin{equation}
     {\rm Tr}[A\Sigma] - \frac{1}{S} {\rm Tr}[A\Sigma] {\rm Tr} [\Lambda A  G_{\mu}^{-1} ] = \frac{1}{S} \sigma^2 {\rm Tr}[\Lambda A G_{\mu}^{-1} ],
\end{equation}
which solves to give
\begin{equation}
    {\rm Tr}[A\Sigma] = \frac{\sigma^2}{S} \frac{ {\rm Tr}[\Lambda AG_{\mu}^{-1}]}{1 -   \frac{1}{S}{\rm Tr} [\Lambda AG_{\mu}^{-1}]} :=  \frac{\sigma^2}{S}\kappa_{\mu}.
\end{equation}
Therefore, $\Sigma$ is 
\begin{equation}
     \Sigma  = \frac{ \sigma^2}{S}\left( 1 + \frac{ \kappa}{S}\right)\Lambda  \left[2(1-\mu)I_D-\left(\frac{1-\mu}{1+\mu}+\frac{1}{S}\right)\Lambda A\right]^{-1}.
\end{equation}
\end{proof}



\subsubsection{Training Error and Test Error for Label Noise}\label{app: errors label}
In the following theorem, we calculate the expected training and test loss for random noise in the label.
\begin{theorem}\label{thm: train and test error of label}$($Approximation error and test loss for SGD noise in the label$)$ The expected approximation error, or the training loss, is defined as $L_{\rm train}:=\mathbb{E}_{\mathbf{w}}[L(\mathbf{w})]$; the expected test loss is defined as $L_{\rm test}:=\frac{1}{2}\mathbb{E}_{\mathbf{w}}\mathbb{E}_{\rm B}\left[(\mathbf{w}^{\rm T}x)^2\right]$. For SGD with noise in the label given by Eq.~\eqref{eq: label noise covariance}, the expected approximation error and test loss are \vspace{-2mm}
\begin{align}
    & L_{\rm train}=\frac{\sigma^2}{2}\left(1+\frac{\lambda \kappa}{S}\right), \label{eq: trainingloss label}\\
    &L_{\rm test}=\frac{\lambda\sigma^2}{2S}\kappa.\label{eq: testloss label}\vspace{-1mm}
\end{align}
\end{theorem}
\begin{remark}
Notably, the training loss decomposes into two additive terms. The term that is proportional to $1$ is the bias, caused by insufficient model expressivity to perfectly fit all the data points, while the second term that is proportional to $\lambda \kappa/S$ is the variance in the model parameter, induced by the randomness of minibatch noise.
\end{remark}
\begin{remark}
When the learning rate $\lambda$ is vanishingly small, the expected test loss diminishes whereas the training error remains finite as long as label noise exists.
\end{remark}

\begin{proof}
We first calculate the approximation error. By definition, 
\begin{align}
    L_{\rm train}&:=\mathbb{E}_{\mathbf{w}}[L(\mathbf{w})] \nonumber\\
    &=  \frac{1}{2}{\rm Tr}[A\Sigma]+\frac{1}{2}\sigma^2 =  \frac{1}{2} \frac{\lambda\sigma^2}{S}\kappa+\frac{1}{2}\sigma^2\nonumber\\
    & =\frac{\sigma^2}{2}\left(1+\frac{\lambda \kappa}{S}\right). 
\end{align}
The test loss is
\begin{align}
    L_{\rm test}&=\frac{1}{2}\mathbb{E}_{\mathbf{w}}\left[\mathbf{w}^{\rm T}A\mathbf{w}\right] =\frac{1}{2}{\rm Tr}[A\Sigma]\nonumber\\
    &= \frac{\lambda\sigma^2}{2S}\kappa.
\end{align}
\end{proof}

\subsection{Minibatch Noise for Random Noise in the Input}\label{app: proofs input}

\subsubsection{Noise Structure}\label{sec: input noise}


Similar to label noise, noise in the input data can also cause fluctuation. 
We assume that the training data points $\tilde{x}_i = x_i + \eta_i$ can be decomposed into a signal part and a random part. As before, we assume Gaussian distributions, $x_i \sim \mathcal{N}(0, A)$ and $\eta_i \sim \mathcal{N}(0, B)$. 
The problem remains analytically solvable if we replace the Gaussian assumption by the weaker assumption that the fourth-order moment exists and takes some matrix form. For conciseness, we assume that there is no noise in the label, namely $y_i=\mathbf{u}^{\rm T}x_i$ with a constant vector $\mathbf{u}$. One important quantity in this case will be $\mathbf{u} \mathbf{u}^{\rm T}  :=  U$. Notice that the trick $\mathbf{w} - \mathbf{u} = \mathbf{w}$ no more works, and so we write the difference explicitly here. The loss function then takes the form \vspace{-0mm}
\begin{align}\label{eq: loss input}
    L(\mathbf{w})&=\frac{1}{2N}\sum_{i=1}^{N}\left[(\mathbf{w}-\mathbf{u})^{\rm T}x_i + \mathbf{w}^{\rm T}\eta_i\right]^2=\frac{1}{2}(\mathbf{w}-\mathbf{u})^{\rm T}A(\mathbf{w}-\mathbf{u})+\frac{1}{2}\mathbf{w}^{\rm T}B\mathbf{w}.\vspace{-1mm}
\end{align}
The gradient $\nabla L=(A+B)(\mathbf{w}-\mathbf{u})+B\mathbf{u}$ vanishes at $\mathbf{w}_*:=(A+B)^{-1}A\mathbf{u}$, which is the minimum of the loss function and the expectation of the parameter at convergence. It can be seen that, even at the minimum $\mathbf{w}_*$, the loss function remains finite unless $\mathbf{u}=0$, which reflects the fact that in the presence of input noise, the network is not expressive enough to memorize all the information of the data. The SGD noise covariance for this type of noise is calculated in the following proposition.


\begin{proposition}\label{prop: noise input}$($Covariance matrix for SGD noise in the input$)$ Let the algorithm be updated according to Eq.~\eqref{eq: with replacement} or \eqref{eq: without replacement} with random noise in the input while the limit $N\to \infty$ is taken with $D$ held fixed. Then the noise covariance is\vspace{-1mm}
\begin{equation}
    C=\frac{1}{S} \left\{ K\Sigma K+ {\rm Tr}[K\Sigma]K + {\rm Tr}[AK^{-1}BU]K \right \},\label{eq: C of input}\vspace{-1mm}
\end{equation}
where $K:=A+B$, and $\Sigma:=\mathbb{E}_{\mathbf{w}}\left[(\mathbf{w}-\mathbf{w}_*)(\mathbf{w}-\mathbf{w}_*)^{\rm T}\right]$.
\end{proposition}
\begin{remark}
 It can be seen that the form of the covariance \eqref{eq: C of input} of input noise is similar to that of label noise \eqref{eq: label noise covariance} with replacing $A$ by $K$ and $\sigma^2$ by ${\rm Tr}[AK^{-1}BU]$, suggesting that these two types of noise share a similar nature. 
\end{remark}

Defining the test loss as $L_{\rm test}:=\frac{1}{2}\mathbb{E}_{\mathbf{w}} \mathbb{E}_{\rm B} \left[(\mathbf{w}^{\rm T}x-\mathbf{u}^{\rm T}x)^2\right]$, Proposition~\ref{prop: noise input} can then be used to calculate the test loss and the model fluctuation.
\begin{theorem}\label{thm: errors of input}$($Training error, test loss and model fluctuation for noise in the input$)$ The expected training loss is defined as $L_{\rm train}:=\mathbb{E}_{\mathbf{w}}[L(\mathbf{w})]$, and the expected test loss is defined as $L_{\rm test}:=\frac{1}{2}\mathbb{E}_{\mathbf{w}} \mathbb{E}_{\rm B} \left[(\mathbf{w}^{\rm T}x-\mathbf{u}^{\rm T}x)^2\right]$. For SGD with noise in the input given in Proposition~\eqref{prop: noise input}, the expected approximation error and test loss are
\begin{align}
    &L_{\rm train}=\frac{1}{2}{\rm Tr}[AK^{-1}BU]\left(1+\frac{\lambda }{S}\kappa'\right), \label{eq: training loss of input}\\
    & L_{\rm test}=\frac{\lambda}{2S}{\rm Tr}[AK^{-1}BU] \kappa'+\frac{1}{2}{\rm Tr}[B'^{\rm T}AB'U], 
\end{align}
where $\kappa':=\frac{ {\rm Tr}[KG'^{-1}]}{1 -  \lambda \frac{1}{S}{\rm Tr} [KG'^{-1}]}$ with $G':=2I_D - \lambda\left(1 + \frac{1}{S} \right)K$, and $B':=K^{-1}B$. Moreover, let $[K,U]=0$. Then the covariance matrix of model parameters is
\begin{align}
    \Sigma=\frac{\lambda {\rm Tr}[AK^{-1}BU]}{S}\left( 1 + \frac{\lambda  \kappa'}{S}\right)  \left[2 I_D - \lambda \left(1 + \frac{1}{S}\right)K \right]^{-1}.
\end{align}
\end{theorem}
\begin{remark}
Note that $[K,U]=0$ is necessary only for an analytical expression of $\Sigma$. It can be obtained by solving Eq.~\eqref{eq: matrix eq of input} even without invoking $[K,U]=0$. In general, the condition that $[K,U]=0$ does not hold. Therefore, only  the training and test error can be calculated exactly.
\end{remark}
\begin{remark}
 The test loss is always smaller than or equal to the training loss because all matrices involved here are positive semidefinite.
\end{remark}

\subsubsection{Proof of Proposition~\ref{prop: noise input}}\label{app: der of noise input}

\begin{proof}
We define $\Sigma:=\mathbb{E}_{\mathbf{w}}\left[(\mathbf{w}-\mathbf{w}_*)(\mathbf{w}-\mathbf{w}_*)^{\rm T}\right] $. Then,
\begin{align}
    &\mathbb{E}_{\mathbf{w}}[\mathbf{w}\mathbf{w}^{\rm T}]=\Sigma+(A+B)^{-1}AUA(A+B)^{-1}:=\Sigma+A'UA'^{\rm T}:=\Sigma_A,\\
    &\mathbb{E}_{\mathbf{w}}\left[(\mathbf{w}-\mathbf{u})(\mathbf{w}-\mathbf{u})^{\rm T}\right]=\Sigma+B'UB'^{\rm T}:=\Sigma_B,
\end{align}
where we use the shorthand notations $A':=(A+B)^{-1}A$, $B':=(A+B)^{-1}B$ and $\Sigma_{A}:=\Sigma+A'UA'^{\rm T}$, $\Sigma_{B}:=\Sigma+B'UB'^{\rm T}$. We remark that the covariance matrix $\Sigma$ here still satisfies the matrix equation \eqref{eq: preconditioning matrix eq} with the Hessian being $K:=A+B$.

The noise covariance is 
\begin{align}
   C(\mathbf{w}) &= \frac{1}{S} \frac{1}{N}\sum_i^N(\mathbf{w}^{\rm T}\tilde{x}_i - \mathbf{u}^{\rm T}{x}_i )\tilde{x}_i \tilde{x}_i^{\rm T}(\mathbf{w}^{\rm T}\tilde{x}_i - \mathbf{u}^{\rm T}{x}_i )^{\rm T} - \frac{1}{S}\nabla L(\mathbf{w})\nabla L(\mathbf{w})^{\rm T} \nonumber\\ 
   &= \frac{1}{S} \big\{ A(\mathbf{w} -\mathbf{u}) (\mathbf{w}-\mathbf{u})^{\rm T} A + 
   B\mathbf{w}\mathbf{w}^{\rm T} B + A(\mathbf{w}-\mathbf{u})\mathbf{w}^{\rm T}B+B\mathbf{w}(\mathbf{w}-\mathbf{u})^{\rm T}A \nonumber\\ 
   &\qquad\ + {\rm Tr}[A(\mathbf{w} -\mathbf{u}) (\mathbf{w}-\mathbf{u})^{\rm T} ]K  + {\rm Tr}[B \mathbf{w}\mathbf{w}^{\rm T} ]K  \big\}.\label{eq: C of input large number}
\end{align}

In Eq.~\eqref{eq: C of input large number}, there are four terms without trace and two terms with trace. We first calculate the traceless terms. For the latter two terms, we have
\begin{align}
    &\mathbb{E}_{\mathbf{w}}[(\mathbf{w}-\mathbf{u})\mathbf{w}^{\rm T}]=\Sigma-A'UB'^{\rm T},\\
    &\mathbb{E}_{\mathbf{w}}[\mathbf{w}(\mathbf{w}-\mathbf{u})^{\rm T}]=\Sigma-B'UA'^{\rm T}.
\end{align}
Because $A'+B'=I_D$, after simple algebra the four traceless terms result in $2(A+B)\Sigma(A+B)$.

The two traceful terms add to ${\rm Tr}[A\Sigma_B +B\Sigma_A]K$. With the relation $AB'=BA'$, what inside the trace is
\begin{align}
    A\Sigma_B +B\Sigma_A=K\Sigma+AK^{-1}BU.
\end{align}

Therefore, the asymptotic noise is
\begin{align}
    C&:=\mathbb{E}_{\mathbf{w}}[C(\mathbf{w})]\nonumber\\ 
    &=\frac{1}{S} \left\{ K\Sigma K+ {\rm Tr}[A\Sigma_B +B\Sigma_A]K  \right \}\\
    &=\frac{1}{S} \left\{ K\Sigma K+ {\rm Tr}[K\Sigma]K + {\rm Tr}[AK^{-1}BU]K \right \}.
\end{align}
\end{proof}

\subsubsection{Proof of Theorem~\ref{thm: errors of input}}\label{app: der of errors of input}

\begin{proof}
The matrix equation satisfied by $\Sigma$ is
\begin{align}
    \Sigma K+K\Sigma  - \lambda\left(1+\frac{1}{S}\right)K\Sigma K =\frac{\lambda}{S}\left({\rm Tr}[K\Sigma] K +{\rm Tr}[AK^{-1}BU]K\right).\label{eq: matrix eq of input}
\end{align}
By using a similar technique as in Appendix~\ref{app: der of Sigma label}, the trace ${\rm Tr}[K\Sigma]$ can be calculated to give
\begin{align}
    {\rm Tr}[K\Sigma]=\frac{\lambda {\rm Tr}[AK^{-1}BU]}{S}\kappa', \label{eq: trace of input}
\end{align}
where $\kappa':=\frac{ {\rm Tr}[KG'^{-1}]}{1 -  \lambda \frac{1}{S}{\rm Tr} [KG'^{-1}]}$ with $G':=2I_D - \lambda\left(1 + \frac{1}{S} \right)K$. 

With Eq.~\eqref{eq: trace of input}, the training error and the test error can be calculated. The approximation error is
\begin{align}
    L_{\rm train}&=\mathbb{E}_{\mathbf{w}}[L(\mathbf{w})]=\frac{1}{2}{\rm Tr}[A\Sigma_B +B\Sigma_A]=\frac{1}{2}{\rm Tr}[AK^{-1}BU]\left(1+\frac{\lambda }{S}\kappa'\right).
\end{align}

The test loss takes the form of a bias-variance tradeoff:
\begin{align}
    L_{\rm test}&=\frac{1}{2}\mathbb{E}_{\mathbf{w}} \mathbb{E}_{\rm B} \left[(\mathbf{w}^{\rm T}x-\mathbf{u}^{\rm T}x)^2\right]=\frac{1}{2}\mathbb{E}_{\mathbf{w}}\left[(\mathbf{w}-\mathbf{u})^{\rm T}A(\mathbf{w}-\mathbf{u})\right]=\frac{1}{2}{\rm Tr}[A\Sigma_B]\nonumber\\
    &=\frac{\lambda}{2S}{\rm Tr}[AK^{-1}BU]\left( 1 + \frac{\lambda  \kappa'}{S}\right){\rm Tr}[AG'^{-1}]+\frac{1}{2}{\rm Tr}[B'^{\rm T}AB'U]\nonumber\\
    &= \frac{\lambda}{2S}{\rm Tr}[AK^{-1}BU] \kappa'+\frac{1}{2}{\rm Tr}[B'^{\rm T}AB'U].
\end{align}

Let $[K,U]=0$. Then $\Sigma$ can be explicitly solved because it is a function of $K$ and $U$. Specifically, 
\begin{align}
    \Sigma=\frac{\lambda {\rm Tr}[AK^{-1}BU]}{S}\left( 1 + \frac{\lambda  \kappa'}{S}\right)  \left[2 I_D - \lambda \left(1 + \frac{1}{S}\right)K \right]^{-1}.
\end{align}
\end{proof}



\subsection{Proofs in Sec.~\ref{sec: regularization}}\label{app: proofs regulariion}
\subsubsection{Proof of Proposition~\ref{prop: C of regular}}\label{app: der of C of regular}

\begin{proof}
The covariance matrix of the noise is 
\begin{align}
   C(\mathbf{w}) &= \frac{1}{S} \frac{1}{N}\sum_i^N\left[(\mathbf{w}-\mathbf{u})^{\rm T}x_i x_i+ \Gamma\mathbf{w}\right]\left[x_i^{\rm T} x_i^{\rm T}(\mathbf{w}-\mathbf{u})+ \mathbf{w}^{\rm T}\Gamma\right] - \frac{1}{S}\nabla L_\Gamma(\mathbf{w})\nabla L_\Gamma(\mathbf{w})^{\rm T} \nonumber\\ 
   &= \frac{1}{S} \big\{ A(\mathbf{w} -\mathbf{u}) (\mathbf{w}-\mathbf{u})^{\rm T} A + 
    {\rm Tr}[A(\mathbf{w} -\mathbf{u}) (\mathbf{w}-\mathbf{u})^{\rm T} ]A \big\}.\label{eq: C of regular large number}
\end{align}

Using a similar trick as in Appendix~\ref{app: der of noise input}, the asymptotic noise is
\begin{align}
    C=\frac{1}{S}\left(A\Sigma A+{\rm Tr}[A\Sigma]A+{\rm Tr}[\Gamma'^{\rm T} A\Gamma' U]A + \Gamma A' U A' \Gamma\right).
\end{align}
\end{proof}


\subsubsection{Proof of Theorem~\ref{thm: errors of regular}}\label{app: der of errors of regular}
Besides the test loss and the model fluctuation, we derive the approximation error here as well.
\begin{theorem*}\label{thm: train error of regular}$($Training error, test loss and model fluctuation for learning with L$_2$ regularization$)$ The expected training loss is defined as $L_{\rm train}:=\mathbb{E}_{\mathbf{w}}[L(\mathbf{w})]$, and the expected test loss is defined as $L_{\rm test}:=\frac{1}{2}\mathbb{E}_{\mathbf{w}} \mathbb{E}_{\rm B} \left[(\mathbf{w}^{\rm T}x-\mathbf{u}^{\rm T}x)^2\right]$. For noise induced by L$_2$ regularization given in Proposition~\ref{prop: C of regular}, let $[A,\Gamma]=0$. Then the expected approximation error and test loss are
\begin{align}
    &L_{\rm train}=\frac{\lambda}{2S}{\rm Tr}[AK^{-2}\Gamma^2 U ]{\rm Tr}[AG^{-1}]\left(1+\frac{\lambda \kappa}{S}\right)+\frac{\lambda}{2S}\left({\rm Tr}[A^2 K^{-2}\Gamma^2 G^{-1}U]+\frac{\lambda r}{S}{\rm Tr}[AG^{-1}]\right)\nonumber\\
    &\qquad \qquad+\frac{1}{2}{\rm Tr}[AK^{-1}\Gamma U],\label{eq: regularization training loss}\\
    &L_{\rm test}=\frac{\lambda}{2S}\left({\rm Tr}[AK^{-2}\Gamma^2 U ]\kappa+r\right)+\frac{1}{2}{\rm Tr}[AK^{-2}\Gamma^2U],
\end{align}
where $\kappa:=\frac{{\rm Tr}[A^2 K^{-1}G^{-1}]}{1-\frac{\lambda}{S}{\rm Tr}[A^2 K^{-1}G^{-1}]}$, $r:=\frac{{\rm Tr}[A^3 K^{-3} \Gamma^2 G^{-1}U]}{1-\frac{\lambda}{S}{\rm Tr}[A^2 K^{-1} G^{-1}]}$, with $G:=2I_D-\lambda\left(K+\frac{1}{S}K^{-1}A^2\right)$. Moreover, if $A$, $\Gamma$ and $U$ commute with each other, the model fluctuation is
\begin{align}
    \Sigma=\frac{\lambda}{S}{\rm Tr}[AK^{-2}\Gamma^2U]\left(1+\frac{\lambda \kappa}{S}\right)AK^{-1}G^{-1}+\frac{\lambda}{S}\left(A^2 K^{-2}\Gamma^2 U+\frac{\lambda r}{S}A \right)K^{-1}G^{-1}.
\end{align}
\end{theorem*}
\begin{remark}
Because $\Gamma$ may not be positive semidefinite, the test loss can be larger than the training loss, which is different from the input noise case.
\end{remark}

\begin{proof}
The matrix equation obeyed by $\Sigma$ is
\begin{align}
    \Sigma K + K \Sigma -\lambda K\Sigma K -\frac{\lambda}{S}A \Sigma A=\frac{\lambda}{S}{\rm Tr}[A\Sigma]A+\frac{\lambda}{S}\left({\rm Tr}[AK^{-2}\Gamma^2 U]A + AK^{-1}\Gamma U \Gamma K^{-1} A\right), \label{eq: matrixeq of regular}
\end{align}
where we use the shorthand notation $K:= A+\Gamma$. Let $[A,\Gamma]=0$. Using the trick in Appendix~\ref{app: der of Sigma label}, the trace term ${\rm Tr}[A\Sigma]$ is calculated as
\begin{align}
    {\rm Tr}[A\Sigma]=\frac{\lambda}{S}\left({\rm Tr}[AK^{-2} \Gamma^2 U]\kappa + r\right),
\end{align}
where $\kappa:=\frac{{\rm Tr}[A^2 K^{-1}G^{-1}]}{1-\frac{\lambda}{S}{\rm Tr}[A^2 K^{-1}G^{-1}]}$, $r:=\frac{{\rm Tr}[A^3 K^{-3} \Gamma^2 G^{-1}U]}{1-\frac{\lambda}{S}{\rm Tr}[A^2 K^{-1} G^{-1}]}$, and $G:=2I_D-\lambda\left(K+\frac{1}{S}K^{-1}A^2\right)$.


The training error is 
\begin{align}
    L_{\rm train}&=\frac{1}{2}{\rm Tr}[A\Sigma_{\Gamma} +\Gamma\Sigma_A]\nonumber\\
    &=\frac{1}{2}{\rm Tr}[K\Sigma] + \frac{1}{2}{\rm Tr}[AK^{-1}\Gamma U]\nonumber\\
    &=\frac{\lambda}{2S}{\rm Tr}[AK^{-2}\Gamma^2 U ]{\rm Tr}[AG^{-1}]\left(1+\frac{\lambda \kappa}{S}\right)+\frac{\lambda}{2S}\left({\rm Tr}[A^2 K^{-2}\Gamma^2 G^{-1}U]+\frac{\lambda r}{S}{\rm Tr}[AG^{-1}]\right)\nonumber\\&\quad+\frac{1}{2}{\rm Tr}[AK^{-1}\Gamma U].
\end{align}
The test loss is 
\begin{align}
    L_{\rm test}&=\frac{1}{2}\mathbb{E}_{\mathbf{w}} \mathbb{E}_{\rm B} \left[(\mathbf{w}^{\rm T}x-\mathbf{u}^{\rm T}x)^2\right]\nonumber\\
    &=\frac{1}{2}\mathbb{E}_{\mathbf{w}}\left[(\mathbf{w}-\mathbf{u})^{\rm T}A(\mathbf{w}-\mathbf{u})\right]=\frac{1}{2}{\rm Tr}[A\Sigma_{\Gamma}]\nonumber\\
    &= \frac{\lambda}{2S}\left({\rm Tr}[AK^{-2}\Gamma^2 U ]\kappa+r\right) +\frac{1}{2}{\rm Tr}[AK^{-2}\Gamma^2U].
\end{align}

Let $A$, $\Gamma$ and $U$ commute with each other. Then,
\begin{align}
    \Sigma=\frac{\lambda}{S}{\rm Tr}[AK^{-2}\Gamma^2U]\left(1+\frac{\lambda \kappa}{S}\right)AK^{-1}G^{-1}+\frac{\lambda}{S}\left(A^2 K^{-2}\Gamma^2 U+\frac{\lambda r}{S}A \right)K^{-1}G^{-1}.
\end{align}
\end{proof}

\subsubsection{Proof of Corollary~\ref{cor: negative gamma}}\label{app: proof of negative gamma}
For a 1d example, the training loss and the test loss have a simple form. We use lowercase letters for 1d cases.
\begin{corollary}
For a 1d SGD with L$_2$ regularization, the training loss and the test loss are
\begin{align}
    &L_{\rm train}=\frac{a\gamma}{2(a+\gamma)}\frac{2(a+\gamma)-\lambda\left[(a+\gamma)^2+\frac{2}{S}a(a-\gamma)\right]}{2(a+\gamma)-\lambda\left[(a+\gamma)^2+\frac{2}{S}a^2\right]}u^2,\label{eq: 1d train regular}\\
    &L_{\rm test}=\frac{a\gamma^2}{2(a+\gamma)}\frac{2-\lambda(a+\gamma)}{2(a+\gamma)-\lambda\left[(a+\gamma)^2+\frac{2}{S}a^2\right]}u^2.\label{eq: 1d test regular}
\end{align}
\end{corollary}

\begin{proof}
 The training error and the test loss for 1d cases can be easily obtained from Theorem~\ref{thm: train error of regular}. 
\end{proof}
Now we prove Corollary~\ref{cor: negative gamma}.
\begin{proof}
The condition for convergence is 
$1-\frac{\lambda}{S}{\rm Tr}[A^2 K^{-1}G^{-1}]>0$. Specifically,
\begin{align}
    \lambda(a+\gamma)^2-2(a+\gamma)+\lambda\frac{2}{S}a^2<0.
\end{align}
For a given $\gamma$, the learning rate needs to satisfy
\begin{align}
    \lambda<\frac{2(a+\gamma)}{(a+\gamma)^2+\frac{2}{S}a^2}.
\end{align}
For a given $\lambda$,  $\gamma$ needs to satisfy
\begin{align}
    \frac{1-a\lambda-\sqrt{1-\frac{2}{S}a^2 \lambda^2}}{\lambda}<\gamma<\frac{1-a\lambda+\sqrt{1-\frac{2}{S}a^2 \lambda^2}}{\lambda},\label{eq: gamma condition}
\end{align}
which indicates a constraint on $\lambda$:
\begin{align}
    a\lambda<\sqrt{\frac{S}{2}}.\label{eq: constriant 1}
\end{align}

If $\gamma$ is allowed to be non-negative, the optimal value can only be $0$ due to the convergence condition. Therefore, a negative optimal $\gamma$ requires an upper bound on it being negative, namely
\begin{align}
    \frac{1-a\lambda+\sqrt{1-\frac{2}{S}a^2 \lambda^2}}{\lambda}<0.
\end{align}
Solving it, we have
\begin{align}
    a\lambda>\frac{2}{1+\frac{2}{S}}. 
\end{align}
By combining with Eq.~\eqref{eq: constriant 1}, a necessary condition for the existence of a negative optimal $\gamma$ is
\begin{align}
    \frac{2}{1+\frac{2}{S}}<\sqrt{\frac{S}{2}} \to (S-2)^2>0 \to S\ne 2.
\end{align}

Hence, a negative optimal $\gamma$ exists, if and only if
\begin{align}
    \frac{2}{1+\frac{2}{S}}<a\lambda<\sqrt{\frac{S}{2}},{\textup{ and}}\ S\ne 2.
\end{align}



\end{proof}

For higher dimension with $\Gamma=\gamma I_D$, it is possible to calculate the optimal $\gamma$ for minimizing the test loss \eqref{eq: regular test loss} as well. Specifically, the condition is given by
\begin{align}
    \frac{d}{d\gamma}L_{\rm test}:=\frac{1}{2}\frac{d}{d\gamma}\frac{f(\gamma)}{g(\gamma)}=0,
\end{align}
where
\begin{align}
    &f(\gamma):=\gamma^2{\rm Tr}\left[AK^{-2}\left(I_D + \frac{\lambda}{S}A^2 K^{-1} G^{-1}\right)U\right],\\
    &g(\gamma):=1-\frac{\lambda}{S}{\rm Tr}[A^2 K^{-1} G^{-1}].
\end{align}
Although it is impossible to solve the equation analytically, it can be solved numerically.

\end{document}